\documentclass[11pt]{article}
\usepackage{amsmath,amssymb,amsfonts,amsthm,epsfig}
\usepackage[usenames,dvipsnames]{xcolor}
\usepackage{bm,xspace}
\usepackage{tcolorbox}
\usepackage{cancel}
\usepackage{fullpage}
\usepackage{liyang}
\usepackage{framed}
\usepackage{verbatim}
\usepackage{enumitem}
\usepackage{array}
\usepackage{multirow}
\usepackage{afterpage}
\usepackage{mathrsfs}
\usepackage{pifont} 
\usepackage{chngpage}
\usepackage[normalem]{ulem}
\usepackage{boxedminipage}
\usepackage{caption}
\usepackage{forest}
\usepackage{tikz}
\usetikzlibrary{trees}
\usetikzlibrary{arrows}
\usepackage{algorithm}
\usepackage{algorithmicx}
\usepackage{algpseudocode}

\usepackage{natbib} %

\usepackage{pgfplots}
\usepackage{microtype}
\pgfplotsset{width=8cm,compat=newest}

\newtheorem*{theorem*}{Theorem}

\def\colorful{1}

\ifnum\colorful=1

\fi
\ifnum\colorful=0

\fi

\newcommand{\TV}{\dist_{\mathrm{TV}}}

\newcommand{\error}{\mathrm{error}}

\newcommand{\Unif}{\mathrm{Unif}}

\newcommand{\depth}{\mathrm{depth}}
\newcommand{\train}{\mathrm{train}}
\newcommand{\test}{\mathrm{test}}

\newcommand{\remaining}{\mathrm{remaining}}
\newcommand{\FindSubcube}
{\textsc{FindSubcubeList}}
\newcommand{\TreeLearn}
{\textsc{TreeLearn}}
\newcommand{\FindTree}
{\textsc{FindTree}}

\newcommand{\bmcH}
{\boldsymbol{\mcH}}

\newcommand{\epsa}
{\eps_{\mathrm{additional}}}

\DeclarePairedDelimiter\ceil{\lceil}{\rceil}

\DeclarePairedDelimiter\abs{|}{|}

\DeclarePairedDelimiter\paren{(}{)}
\DeclarePairedDelimiter\bracket{[}{]}
\DeclarePairedDelimiter\sset{\{}{\}}

\newcommand{\iid}{\overset{\mathrm{iid}}{\sim}}

\newlist{enumprop}{enumerate}{1} 
\setlist[enumprop]{label=\arabic*.,ref=\theproposition.\arabic*}
\crefalias{enumpropi}{proposition}

\makeatletter
\newtheorem*{rep@theorem}{\rep@title}
\newcommand{\newreptheorem}[2]{
\newenvironment{rep#1}[1]{
 \def\rep@title{#2 \ref{##1}}
 \begin{rep@theorem}\itshape}
 {\end{rep@theorem}}}
\makeatother

\newreptheorem{theorem}{Theorem}

\newcommand{\pparagraph}[1]{\bigskip \noindent {\bf {#1}}}

\def\coltformat{0}

\begin{document}
\title{A Distributional-Lifting Theorem for PAC Learning 
}

\author{ 
Guy Blanc \vspace{6pt} \\ 
{{\sl Stanford}}  \and 
\hspace{5pt} Jane Lange \vspace{6pt} \\
\hspace{5pt} {{\sl MIT}} \vspace{10pt}  \and 
\hspace{5pt} Carmen Strassle \vspace{6pt} \\
\hspace{5pt} {{\sl Stanford}} \vspace{10pt} \and 
Li-Yang Tan \vspace{6pt}  \\
\hspace{-10pt} {{\sl Stanford}}
}

\date{\vspace{15pt}\small{\today}}

\maketitle

\begin{abstract}
The apparent difficulty of efficient distribution-free PAC learning has led to a large body of work on distribution-specific learning. Distributional assumptions facilitate the design of efficient algorithms but also limit their reach and relevance. Towards addressing this, we prove a {\sl distributional-lifting theorem}: This upgrades a learner that succeeds with respect to a limited distribution family $\mathcal{D}$ to one that succeeds with respect to {\sl any} distribution $D^\star$, with an efficiency overhead that scales with the complexity of expressing $D^\star$ as a mixture of distributions in $\mathcal{D}$. 

Recent work of Blanc, Lange, Malik, and Tan considered the special case of lifting  uniform-distribution learners and designed a lifter that uses a  conditional sample oracle for $D^\star$, a strong form of access not afforded by the standard PAC model. Their approach, which draws on ideas from semi-supervised learning, first learns $D^\star$ and then uses this information to lift. 

We show that their approach is information-theoretically intractable with access only to random examples, thereby giving formal justification for their use of the conditional sample oracle. We then take a different approach that sidesteps the need to learn $D^\star$, yielding a lifter that works in the standard PAC model and enjoys additional advantages: it works for all base distribution families, preserves the noise tolerance of learners, has better sample complexity, and is simpler. \ifnum\coltformat=1 \footnote{Extended abstract. Full version available on arxiv.}\fi \ifnum\coltformat=0 \footnote{Accepted for presentation at the
Conference on Learning Theory (COLT) 2025}\fi

\end{abstract}

\ifnum\coltformat=1
\fi

\ifnum\coltformat=0
\thispagestyle{empty}
\newpage 
\setcounter{page}{1}
\fi

\section{Introduction}

\paragraph{Distribution-free PAC learning.} In PAC learning~\citep{Val84}, the learner is given labeled examples $(\bx,f(\bx))$ where $f$ is an unknown target function and $\bx$ is distributed according to an unknown distribution $D$. While~$f$ is assumed to belong to a known  concept class $\mathcal{C}$, no assumptions are made about $D$. This {\sl distribution-free} aspect of the model is a notable strength: the guarantees of a PAC learning algorithm hold for {\sl all} distributions, regardless of how simple or complicated they may be.

On the other hand, it is also this very aspect of the model that has skewed the balance between positive and negative results heavily towards the latter. Indeed, the introduction of the model quickly led to a flurry of works (e.g.~\cite{KLPV87,PV88,LV88,PW88,BR88,Han89,PW89,KV89}\,...)~showing that even simple concept classes $\mathcal{C}$ can be computationally hard to learn. The unifying idea across all such results is that many canonical hard problems---from factoring to NP-complete problems---can, roughly speaking, be embedded within the distribution~$D$. This therefore shows that the simplicity of $\mathcal{C}$ can be negated by the complexity of $D$.



\paragraph{Distribution-specific PAC learning.} This suggests the distribution-{\sl specific} variant of PAC learning where the underlying distribution $D$ is fixed, known, and ``nice".  Kearns, Li, Pitt, and Valiant~\citep{KLPV87} were the first to  consider this setting, citing it as a way of circumventing the hardness of distribution-free learning. 
Other early works include~\cite{BI91,Nat92}. 



Allowing for distributional assumptions has indeed led to a bounty of positive results, with most of them assuming that $D$ belongs to the class $\mathcal{D}_{\mathrm{prod}}$ of product distributions. 
We now have fast product-distribution learners for basic classes such as DNF formulas~\citep{Ver90,Jac97,GKK08} and intersections of halfspaces~\citep{KOS04,Kane14-intersections}. These contrast with the strong lower bounds that have been shown for these classes in the distribution-free setting~\citep{KS07,KS09,RS10,She10,She13}.  The connection between product-distribution learning and the Fourier analysis of boolean functions has given us general and powerful techniques for designing learners~\citep{LMN93,KM93}. Notably, these techniques extend to the {\sl agnostic} setting~\citep{KSS94,GKK08,KKMS08}, yielding fast and noise-tolerant learners for many expressive classes. This again contrasts with the distribution-free setting where even the agnostic learning of {\sl disjunctions} is well-known to be a difficult open problem.

The main caveat that comes with product-distribution learners is, of course, that the assumption of independence is a strong and stylized one. Much of the richness and complexity of real-world data distributions come from the correlations among features. Therefore, while product-distribution PAC learning has become an important area of study within theoretical computer science, it admittedly departs from Valiant's original intention for the PAC model to serve as a theoretical framework for machine learning.  

\paragraph{The two extremes of distributional assumptions.} Summarizing, much of the work in PAC learning has focused on two extremes in terms of distributional assumptions. At one extreme  we have distribution-free PAC learning where  no assumptions are made about the underlying distribution. Algorithms in this setting come with an extremely strong guarantee, but obtaining such algorithms has been correspondingly challenging. At the other extreme we have distribution-specific  PAC learning. This setting has enjoyed  algorithmic success in the form of fast algorithms and sophisticated techniques, but at the price of a strong, arguably unrealistic, distributional assumption. Research on these settings has resulted in two large and mostly separate islands of work.


\subsection{Prior work of~\cite{BLMT-lifting}}

\paragraph{Interpolating between the extremes via distributional lifting.} Towards addressing this, recent work of Blanc, Lange, Malik, and Tan~\citep{BLMT-lifting} proposed the possibility of  {\sl interpolating} between uniform-distribution learning, the simplest case of product-distribution learning, and distribution-free learning. Their motivation was twofold. First, if the distribution-free setting is overly demanding and the assumption of independence is overly strong, perhaps we should search for fruitful middle grounds that allow for efficient algorithms whose guarantees hold for limited yet still broad classes of distributions. Second, just as PAC learning has benefited from the study of function complexity---under the hood of every new PAC learning algorithm are new insights about the structure of functions in the concept class---perhaps it could also benefit from the equally large literature on  distributional complexity.

\cite{BLMT-lifting} instantiated this proposal with decision tree complexity as the interpolating notion of distributional complexity. The decision tree complexity of a distribution $D^\star$ over $\{\pm 1\}^n$ is the smallest integer~$d$ such that $D^\star$ can be decomposed into a depth-$d$ decision tree of uniform pieces. This is therefore a complexity measure that indeed interpolates between the uniform distribution $(d=0)$ and the class of all distributions $(d=n)$. Small depth decision tree distributions were first considered by Feldman, O'Donnell, and Servedio~\citep{FOS08} in the context of distribution learning. Such distributions generalize junta distributions~\citep{ABR16} and are in turn a special case of mixtures of subcubes~\citep{CM19} and mixtures of product distributions~\citep{FM99,Cry99,CGG01,FOS08}.

\cite{BLMT-lifting}'s main result is a {\sl distributional-lifting} theorem that  transforms, in a blackbox fashion, a uniform-distribution learner into one that works for {\sl any} distribution $D^\star$, with an overhead the scales with the decision tree complexity of $D^\star$: \medskip 

\noindent {\bf \cite{BLMT-lifting}'s lifter.} {\it Let $\mathcal{C}$ be a concept class over $\{\pm 1\}^n$ that is closed under restrictions. Suppose there is an algorithm that learns $\mathcal{C}$ under the uniform distribution in $\poly(n)$ time. There is an algorithm that, given a conditional sample oracle to any depth-$d$ decision tree distribution $D^\star$, learns $\mathcal{C}$ under~$D^\star$ in $n^{O(d)}$ time. 
}\medskip

The conditional sample oracle~\citep{CRS15,BC18,CCKLW21,CJLW21} allows the learner to specify any subset $S\sse \{\pm 1\}^n$ and receive unlabeled samples drawn from the conditional distribution of $D^\star$ over $S$. We will soon say more about this oracle.  Setting it aside for now, we note that the $n^{O(d)}$ runtime of~\citep{BLMT-lifting}'s lifted learner indeed scales with the inherent complexity of $D^*$, retaining the base learner's $\poly(n)$ runtime for constant $d$, and remaining subexponential for $d$ as large as $\Theta(n^{0.99})$. On a more technical level, their result carries the appealing message  that the large technical toolkit for uniform-distribution learning---all of which relies crucially on the assumption of independence among features---now have implications on settings where distributions have highly-correlated features. 


\paragraph{Lifting replicable PAC learners.} Recent work of Kalavasis, Karbasi, Velegkas, and Zhou~\citep{KKVZ24} extended~\cite{BLMT-lifting}'s lifter to the setting of {\sl replicable} PAC learning~\citep{ILPS22}, achieving  similar parameters under the same assumptions.

\subsection{Downsides of~\cite{BLMT-lifting}'s lifter} While~\cite{BLMT-lifting}'s work shows that distributional lifting is a viable concept, their  lifter comes with three important downsides. 

\begin{enumerate}
    \item The {\bf \emph{conditional sample oracle}} is a  strong form of access not afforded by the standard PAC model. In fact,~\cite{BLMT-lifting}'s lifter uses this oracle with a near-maximal ``amount of conditioning": it requests for samples of $D^\star$ conditioned on sets $S$ of constant size, allowing it to receive samples that are extremely unlikely under $D^\star$ itself.

This is a serious  downside as the lifted learner now has active control over the samples that it receives, allowing it to ``home in" on small and specific parts of the distribution. 
The assumption of uniformity is therefore traded for what is arguably an equally unrealistic assumption that the learner has access to such an oracle. Put differently, their result does not perform distributional lifting {\sl within} the PAC model: it takes a random-example PAC learner and lifts it to one that is no longer a random-example PAC learner.\footnote{\cite{BLMT-lifting} and~\cite{KKVZ24} do show that the oracle is not necessary in the case of {\sl monotone} distributions $D^\star$.}

    \item Their lifter {\bf \emph{only lifts uniform-distribution base learners}}. Notably, as we explain in the body of the paper, their lifter is unable to even lift product-distribution base learners. (It is easy to see that most product distributions require the maxmimal decision tree depth $n$ to decompose into uniform pieces.)  This is especially unfortunate since the analyses of uniform-distribution learners tend to extend straightfowardly, sometimes even in a blackbox fashion, to the setting of general product distributions. In particular, Fourier-based uniform-distribution learners all have product-distribution analogues.  

Ideally, one would like to lift base learners for an arbitrary class $\mathcal{D}$ of distributions into one that succeeds under any distribution $D^\star$, with an overhead that scales with the complexity of expressing $D^\star$ as a decision tree of distributions in $\mathcal{D}$. \ifnum\coltformat=1
See \Cref{figure:DT-leaf-dist}.
\fi

\ifnum\coltformat=0
\bigskip 
\fi

\ifnum\coltformat=0
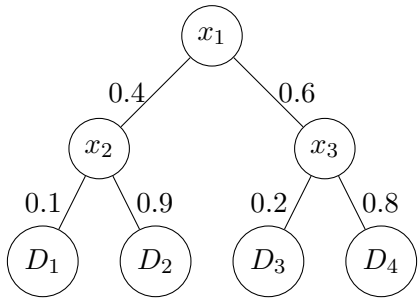
\begin{figure}[h]
    \centering
    \begin{tikzpicture}[level distance=15mm]
        \node[draw, circle] {$x_1$} [sibling distance=30mm]
            child { node[draw, circle] {$x_2$} [sibling distance=15mm]
                child { node[draw, circle] {$D_1$} edge from parent node[left] {$0.1$} }
                child { node[draw, circle] {$D_2$} edge from parent node[right] {$0.9$} }
                edge from parent node[left] {$0.4$} }
            child { node[draw, circle] {$x_3$} [sibling distance=15mm]
                child { node[draw, circle] {$D_3$} edge from parent node[left] {$0.2$} }
                child { node[draw, circle] {$D_4$} edge from parent node[right] {$0.8$} }
                edge from parent node[right] {$0.6$} };
    \end{tikzpicture}
    \caption{A decision tree with leaves $D_i \in \mathcal{D}$.}
    \label{figure:DT-leaf-dist}
\end{figure}
\fi

   \item Their lifter {\bf \emph{does not preserve the noise tolerance of learners}}. That is, even if the base learners are noise-tolerant, the resulting lifted learner is not guaranteed to be.  This is again unfortunate, given the large literature on noise-tolerant learning under product distributions and other distribution-specific settings.

\end{enumerate}

\section{This work}

Our main result is a new lifter that addresses all three downsides:


\begin{tcolorbox}[colback = white,arc=1mm, boxrule=0.25mm]
\vspace{5pt}
\begin{theorem}[Informal; see~\Cref{thm:lift-DT} for the formal version]
\label{thm:main intro}    
Let $\mathcal{C}$ be a concept class and $\mathcal{D}$ be a family of distributions over $\bits^n$ that is closed under restrictions. Suppose there is an algorithm $A$ that learns $\mathcal{C}$ under any $D \in \mathcal{D}$ in $\poly(n)$ time. 

There is an algorithm that, for any $D^\star$ that admits a depth-$d$ decomposition into distributions in $\mathcal{D}$, learns $\mathcal{C}$ under $D^\star$ in $n^{O(d)}$ time and with $2^{O(d)}\cdot \poly(n)$ samples. Furthermore, if the base learner $A$ is noise tolerant, then so is the lifted learner. 
\end{theorem}
\vspace{1pt}
\vspace{-8pt}
\end{tcolorbox}



\paragraph{Overcoming an impossibility result.} Our approach to obtaining~\Cref{thm:main intro} is entirely different from~\cite{BLMT-lifting}'s---necessarily so, as we now explain. Drawing on ideas from semi-supervised learning,~\cite{BLMT-lifting} advocated for a two-step approach to distributional-lifting: 
\begin{itemize}
    \item[] {\sl Step 1:} Learn the decision tree decomposition of $D^\star$ into uniform pieces.
    \item[] {\sl Step 2:} Exploit this knowledge of the decomposition to lift.
\end{itemize}

The bulk of their work goes into designing an algorithm for Step 1---with the decomposition in hand, a lifter follows quite straightforwardly. It is also Step 1 for which they use the conditional sample oracle, and in their paper they pointed to several aspects of their algorithm for which the oracle is necessary. 

The starting point of our work is an impossibility result (\Cref{thm:BLMT-lowerbound}) for the whole two-step approach: We show, via a reduction to uniformity testing, that it is information-theoretically impossible to carry out the first step using only random examples. The novelty of our work is therefore a lifter that, by necessity, sidesteps the need to learn $D^\star$. While the main conceptual message of~\cite{BLMT-lifting}'s work is that techniques from distributional learning can be useful for the distributional lifting of PAC learners, our work provides the complementary message that an algorithm for the former is {\sl not} a prerequisite for the latter. 


We prove this impossibility result in~\Cref{sec:BLMT limitations}, where we also explain why~\cite{BLMT-lifting}'s approach leads to the the second and third downsides listed above.  


\paragraph{Additional advantages of~\Cref{thm:main intro}.} In addition to addressing these downsides, our new approach comes with a couple of additional advantages: 

\begin{enumerate}
    \item {\bf \emph{Improved sample complexity.}} The sample complexity of our lifted learner is $2^{O(d)}\cdot \poly(n)$, a substantial improvement over~\cite{BLMT-lifting}'s  $n^{O(d)}$. In particular, our sample complexity remains $\poly(n)$ for $d = \Theta(\log n)$ whereas theirs becomes $\mathrm{superpoly}(n)$ for any $d = \omega(1)$. 
    
    \item {\bf \emph{Handles general subcube partitions.}}  We are able to handle a complexity measure that is significantly more expressive than decision tree depth. A depth-$d$ decision tree induces a partition of  $\bn$ into subcubes of codimension $d$, but the converse is not true: a partition of $\bn$ into subcubes of codimension $d$ may not be realized by a decision tree. Furthermore, we only require that the subcubes of codimension $d$ partition $\supp(D^\star)$ instead of all of $\bn$. See~\Cref{fig:DT vs subcube}.

\ifnum\coltformat=0
\begin{figure}[h!]
    \centering
\begin{tikzpicture}
    \draw (0, 0) rectangle (3, 3); 
    \draw (0, 1.5) -- (3, 1.5); 
    \draw (1.5, 1.5) -- (1.5, 3); 

    \draw (4, 0) rectangle (7, 3); 
    \draw (4.3, 0.7) rectangle (5.4, 1.5); 
    \draw (5.8, 0.6) rectangle (6.7, 1.4); 
    \draw (5.5, 1.8) rectangle (6.6, 2.8); 
\end{tikzpicture}
    \caption{Decision tree partition of $\bits^n$ (left) vs.~subcube partition of $\supp(D^\star)$ (right)}
    \label{fig:DT vs subcube}
\end{figure}
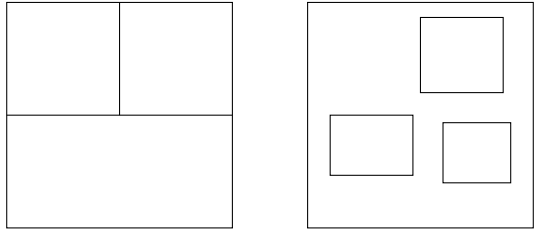
\fi

As we note in \Cref{sec:prelim}, there is an exponential separation in the expressiveness of these measures. It is easy to construct distributions $D^\star$ whose support can be partitioned into uniform distributions over subcubes of codimension $d$, and yet any decision tree decomposition of $D^\star$ into uniform pieces requires depth $2^{\Omega(d/\log d)}$. 

    \cite{BLMT-boosting}'s lifter appears  to rely quite crucially on the hierarchical structure of the partitions induced by a decision tree. Our analysis, on the other hand, extends to handle subcube partitions with the same performance guarantees. This extension requires a few additional ideas beyond those that go into~\Cref{thm:main intro}; we describe them in~\Cref{sec:technical overview}.
\end{enumerate}

\subsection{Other related work}

\paragraph{The work of~\cite{BOW10}.} While~\cite{BLMT-lifting} was the first to explicitly study distributional lifting of PAC learners, the general idea dates back (at least) to the work of Blais, O'Donnell, and Wimmer~\citep{BOW10}. The main result of their paper was an efficient algorithm for performing polynomial regression  under product distributions. As an application, they observed that their algorithm can be used as a base learner in conjunction with a distributional learning algorithm of Feldman, O'Donnell, and Servedio~\citep{FOS08} for learning mixtures of product distributions.  

Like~\cite{BLMT-lifting},~\cite{BOW10}'s observation is based on a two-stage approach of first learning the marginal distribution and then exploiting this information to learn the target function. (This general approach is ubiquitous in semi-supervised learning; see e.g.~\citep{GBDBGTU19} and the references therein.) As discussed, the crux of our work is in departing from such an approach.

\paragraph{Composition theorems in PAC learning.}  Distributional-lifting theorems may be viewed as the distributional analogue of {\sl composition theorems} in PAC learning~\citep{KLPV87,PW90,AKMW96,BDBK97,Bsh98,BM06}. At a high level, these theorems show how learning algorithms for a concept class $\mathcal{C}$ (or a family of concept classes) can be generically lifted to give learning algorithms for larger concept classes.  

\section{Technical Overview}

\label{sec:technical overview}

As mentioned, obtaining a distributional-lifting theorem that uses only random examples necessitates an approach that is different from~\cite{BLMT-lifting}'s. Since the challenges are evident even in their specific setting of lifting uniform-distribution learners,  in this overview we specialize our discussion to that setting. Our approach also extends rather easily from that setting to that of lifting from arbitrary distribution families, so this overview captures the essential ideas. 

\subsection{Quick overview of~\cite{BLMT-lifting}'s approach} We briefly describe~\cite{BLMT-lifting}'s two-stage approach while also introducing notation that will be useful for describing and contrasting our approach. Suppose $D^\star$ admits a decomposition into a small-depth decision tree $T^\star$ of uniform pieces. Let us denote this as $D = T^\star \circ \mathrm{unif}$.  \cite{BLMT-lifting} first uses the conditional sample oracle to learn $D^\star$, i.e.~find a hypothesis tree $T$ such that $\dist_{\mathrm{TV}}(T \circ \mathrm{unif},T^{\star} \circ \mathrm{unif})$ is small\ifnum\coltformat=0
:
\fi
\ifnum\coltformat=1
\ (see \Cref{fig:BLMT}).
\fi

\ifnum\coltformat=0
\begin{center}
\begin{figure}[h]
    \centering
\begin{tikzpicture}
    \draw (0,0) -- (2,4) -- (4,0) -- cycle;
    \node at (2, 1.7) {\LARGE $T$};

    \foreach \x in {0.5, 1.5, 2.5, 3.5} {
        \draw (\x, -0.25) -- (\x, 0.25);
        \node at (\x, -0.5) {\footnotesize Unif};
    }

    \draw (6,0) -- (8,4) -- (10,0) -- cycle;
    \node at (8.2, 1.7) {\LARGE $T^\star$};

    \foreach \x in {6.5, 7.5, 8.5, 9.5} {
        \draw (\x, -0.25) -- (\x, 0.25);
        \node at (\x, -0.5) {\footnotesize Unif};
    }

    \node at (5, 2) {\LARGE $\approx$};
\end{tikzpicture}    
    \caption{First step of~\citep{BLMT-lifting}'s lifter: Learning the decision tree decomposition}
    \label{fig:BLMT}
\end{figure}
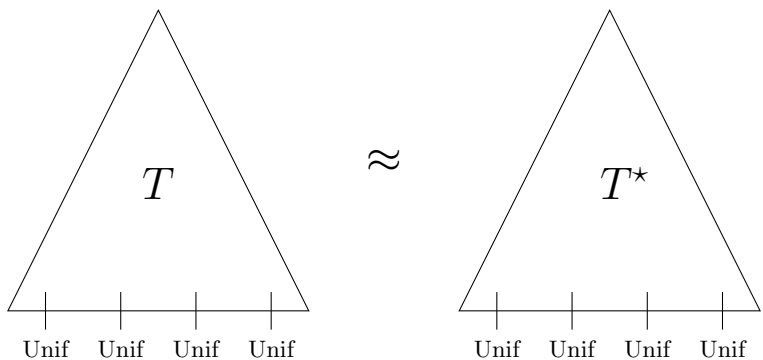
\end{center}
\vspace{-25pt}
\fi

With $T$ in hand, they  draw a training set $\bS$ from $D^\star$. This tree $T$ induces a natural partition of $\bS$ into $\bS = \bS_1 \sqcup \cdots \sqcup \bS_s$ where $s$ is the number of leaves in $T$.  They then run the base learner $A$ on $\bS_1,\ldots,\bS_s$ parts to obtain high-accuracy hypotheses $h_1,h_2,\ldots, h_s$. Their final lifted hypothesis will be $T$ composed with $h_1,h_2,\ldots,h_s$\ifnum\coltformat=0
:
\fi
\ifnum\coltformat=1
\ (see \Cref{fig:tree-composed}).
\fi
\ifnum\coltformat=0
\begin{center}
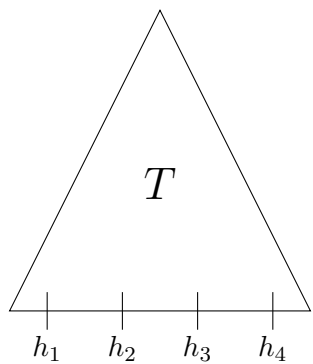
\begin{figure}[h]
\centering
\begin{tikzpicture}
    
\draw (0,0) -- (2,4) -- (4,0) -- cycle;
\node at (2, 1.7) {\LARGE $T$};

\foreach \x [count=\i] in {0.5, 1.5, 2.5, 3.5} {
    \draw (\x, -0.25) -- (\x, 0.25);
    \node at (\x, -0.5) {\normalsize $h_{\i}$};
}
\end{tikzpicture}
\caption{A tree composed with leaf-hypotheses}
\label{fig:tree-composed}
\end{figure}
\end{center}
\fi

Let us denote this hypothesis as $H \coloneqq T \circ (h_1,\ldots,h_s)$. It is straightforward to see that, with an appropriate choice of parameters, $H$ will w.h.p.~be a high-accuracy hypothesis for the target function $f$ with respect to $D^\star$. The simple but crucial fact here is that if $\bS$ is drawn according to $D^\star$, then each of the induced parts $\bS_1,\ldots,\bS_s$ will be distributed near-uniformly. We will return to this point in \Cref{subsec:DT-lifter-overview}. 


\subsection{A first attempt and why it fails}

We know that we cannot hope to learn the decomposition of $D^\star$ from random examples only. Our approach is therefore more direct in nature, by necessity: Rather than first searching for a tree $T$ that is a good approximation of $T^\star$ and then running the base learner $A$ on the partition induced by $T$, we will search directly for a tree $T$  such that running $A$ on the partition induced by $T$ yields a lifted hypothesis that is close to the target function. Indeed, this seems like such a direct approach would even result in an algorithm that is simpler than~\cite{BLMT-lifting}'s.

However, the fact that $T$ may not be a good approximation of $T^\star$, i.e.~that $T$ may not satisfy the property depicted in~\Cref{fig:BLMT}, poses difficulties for the analysis.  To illustrate these difficulties, consider the following candidate algorithm (which {\sl is} much simpler than~\cite{BLMT-lifting}'s):  

\begin{quote}
\begin{tcolorbox}[colback = white,arc=1mm, boxrule=0.25mm]
\begin{center}
\vspace{5pt}
{\bf A first attempt} 
\begin{enumerate}[leftmargin=15pt]
    \item Draw a set $\bS$ of random examples from $D^\star$. 
    \item Search for a tree $T$ such that $A$ returns high-accuracy hypotheses $h_1,\ldots,h_s$ when run on each of the parts in the partition $\bS = \bS_1 \sqcup \cdots \sqcup \bS_s$ induced by~$T$.   \item Return the lifted hypothesis $H = T \circ (h_1,\ldots,h_s)$.
\end{enumerate} 
\end{center}
\end{tcolorbox}
\end{quote}

Note that the  search criterion for $T$ here is indeed laxer than that of~\cite{BLMT-lifting}'s:

\begin{itemize}
    \item [$\Rightarrow$] If $\dist_{\mathrm{TV}}(T^\star \circ \mathrm{unif},T \circ\mathrm{unif})$ is small then $T$ decomposes $D^\star$ into near-uniform pieces. Each part $\bS_1,\ldots,\bS_s$ will therefore be distributed near-uniformly and $A$ will return high-accuracy hypotheses when run on them.  
    \item [$\not\Leftarrow$] Just because $A$ returns high-accuracy hypotheses when run on $\bS_1,\ldots, \bS_s$ does {\sl not} imply that they are distributed near-uniformly.  The base learner $A$ is guaranteed to succeed when run on data drawn from near-uniform distributions, but it could  also succeed on data drawn from distributions that are far from uniform. 
 \end{itemize}

\paragraph{The issue: Having to run $A$ on distributions we have no guarantees for.} When we say that ``$A$ returns a high-accuracy hypothesis when run on $\bS_i$" above, we mean a hypothesis with high accuracy as measured according to $\bS_i$, i.e.~one that achieves high {\sl training} accuracy.  We would be done if we were then able to infer that this hypothesis also achieves high {\sl true} accuracy, i.e.~accuracy as measured according to $D^\star$. \cite{BLMT-lifting} are able to make this inference since they are running $A$ on  near-uniform distributions for which its performance guarantees hold: As long as the size of $\bS_i$ is at least the sample complexity of $A$, its training error will be close to its true error---its hypothesis will generalize. However, in our case we may be running $A$ on data drawn from distributions that are far from uniform, for which its guarantees no longer hold. Indeed, given our impossibility result we know that we {\sl will} be running $A$ on such data---otherwise, we would have learned the decomposition of $D^\star$, which we have shown is impossible.  When this happens, we do not have {\sl any} bound on the number of samples sufficient for generalization.

\subsection{Our actual lifter}
\label{subsec:DT-lifter-overview}

Our actual lifter remains almost as simple as the first attempt:


\begin{quote}
\begin{tcolorbox}[colback = white,arc=1mm, boxrule=0.25mm]
\begin{center}
\vspace{5pt}
{\bf Our actual lifter} 
\begin{enumerate}[leftmargin=15pt]
    \item Draw independent sets $\bS_{\mathrm{train}}$ and $\bS_{\mathrm{test}}$ of random examples from $D^\star$. 
    \item For each candidate tree $T$, let $h_1,\ldots,h_s$ denote the hypotheses returned by $A$   when run on the parts in the partition $\bS_{\mathrm{train}} = (\bS_{\mathrm{train}})_1 \sqcup \cdots \sqcup (\bS_{\mathrm{train}})_s$ induced by $T$.    \item Return the lifted hypothesis $H = T \circ (h_1,\ldots,h_s)$ with the highest accuracy on $\bS_{\mathrm{test}}$.
\end{enumerate} 
\end{center}
\end{tcolorbox}
\end{quote}

We implement this lifter in a way that avoids exhaustively searching through the roughly $n^{2^d}$ candidate depth-$d$ trees over $n$ variables, and instead runs in $n^{O(d)}$ time (see \Cref{sec:time-efficiency}).

\paragraph{Overview of our analysis.} Rather than showing that each of the hypotheses of $A$ generalizes (as~\cite{BLMT-lifting} does),  we directly show that our lifted hypothesis $H$ generalizes.  To describe our analysis, we need to be more explicit about how our lifter depends on the samples it draws: For a tree $T$ and a set $S$ of samples, we write $T_S$ to denote $T \circ (h_1,\ldots,h_s)$ where $h_1,\ldots,h_s$ are the hypotheses that $A$ returns when run on the parts in the partition of $S$ induced by $T$. 

There are three components to our analysis: 

\begin{enumerate}
    \item {\bf Our hypothesis minimizes test loss:} By definition, our lifter returns the hypothesis \[ \text{$T_{\bS_{\mathrm{train}}}$ that minimizes test loss, $\error(T_{\bS_{\mathrm{train}}},\bS_\test)$.}\]  
    \item {\bf True tree has low true loss:} The true tree $T^\star$ induces a partition of $\bS_{\mathrm{train}}$ into parts that are each near-uniformly distributed. If each of these parts is sufficiently large---at least $m$, the sample complexity of $A$---the hypotheses returned by $A$ will generalize. Therefore, as long as $\bS_{\mathrm{train}}$ is sufficiently large, \[  \text{$T^\star_{\bS_{\mathrm{train}}}$ will have low true loss, $\error(T^\star_{\bS_{\mathrm{train}}},D^\star)$.}\]
    \item {\bf Test loss $\approx$ true loss:} Finally, we show that if $\bS_{\mathrm{test}}$ is sufficiently large, the test loss will be close to the true loss: 
    \[ \error(\,\cdot\,,\bS_{\test}) \approx \error(\,\cdot\,,D^\star). \]
    Applying this twice, once to $T_{\bS_{\mathrm{train}}}$ and once to $T^\star_{\bS_{\mathrm{train}}}$, yields the desired conclusion that  our lifted hypothesis achieves low true loss. 
    
Since we would like to apply this to $T_{\bS_{\mathrm{train}}}$ where $T$ is {\sl not}  guaranteed to partition $\bS_{\mathrm{train}}$ into near-uniform pieces, we cannot rely on the sample complexity guarantee of $A$ to reason about how large $\bS_{\mathrm{test}}$ has to be. To get around this, we use the fact that each tree $T$ {\sl determines} $A$'s behavior on the partitions it induces. Our analysis in fact shows that $\bS_{\mathrm{test}}$'s size actually only needs to depend on the number of possible trees, not on $A$'s sample complexity. 
\end{enumerate}
We remark that this approach of using a training set to build a ``cover" guaranteed to have at least one good hypothesis, and a test set to select the best hypothesis is fairly standard; see e.g.~\cite{HKLM22} and the references therein. The key innovation in our setting is, for the particular cover we build, we are able to search for the best hypothesis in a computationally efficient manner.

\subsection{Extension to subcube partitions}
Our approach can handle distributions that decompose into subcubes but do not have the hierarchical 
structure of a decision tree. In this setting,
we assume that the support of $\mcD$ can be decomposed into $s$ disjoint
subcubes, each of codimension $\le d$, such that the distribution over each subcube is in the base distribution
family (see \Cref{def:subcube-decomp}). The high-level structure of our algorithm is similar to the decision tree lifter described in \Cref{subsec:DT-lifter-overview}.

\begin{quote}
\begin{tcolorbox}[colback = white,arc=1mm, boxrule=0.25mm]
\begin{center}
\vspace{5pt}
{\bf Our lifter for subcube partitions} 
\begin{enumerate}[leftmargin=15pt]
    \item Draw independent sets $\bS_{\mathrm{train}}$ and $\bS_{\mathrm{test}}$ of random examples from $D^\star$. 
	\item Run $\mcA$ on all possible subcubes of codimension $\le d$, returning a hypothesis for each.
    \item Run the approximation algorithm described in \Cref{lem:subcube-approx-intro} to find a hypothesis with low test error.
\end{enumerate} 
\end{center}
\end{tcolorbox}
\end{quote}
\paragraph{New technical ingredients for this setting.}
However, the decision tree setting provides a convenience we cannot use in the subcube partitions setting:
the existence of a $n^{O(d)}$-time recursive search process that minimizes test loss over all depth-$d$ candidate trees.
Without this nice bit of structure, it is not clear that one can do any better than enumerating all 
$\approx \binom{n^d}{s}$ possible subcube partitions.
In \Cref{lem:subcube-approx-intro}, we circumvent this issue by relaxing the problem in the following ways:
\begin{itemize}
		\item[$\circ$] \textbf{A more expressive hypothesis:} 
Though our assumption is that the distribution has a ground-truth subcube partition,
we return a global hypothesis where the individual subcube hypotheses are not indexed
by a subcube partition; rather, an ordered list of subcubes such that each $x \in \bits^n$ maps to the first subcube in the list consistent with it.
This ``list hypothesis'' need not have subcubes that are disjoint, and it also
may be of length up to $s \log(1/\eps)$ rather than $s$.
		\item[$\circ$] \textbf{Approximate test error minimization:} 
On each restriction, we run $\mcA$ on $S_\train$ to error $\frac{\eps}{\log(1/\eps)}$. This means that
the ground truth partition would have true loss $\frac{\eps}{\log(1/\eps)}$, however,
we return a hypothesis of true loss $\eps$.
\end{itemize}
With these relaxations, we are able to prove the following.
\begin{lemma}[Informal version of \Cref{lem:subcube-approx}]
    \label[lemma]{lem:subcube-approx-intro}
    There is an efficient algorithm that takes as input a test set $S_{\test}$ and hypotheses $h_{\rho}$ for every restriction $|\rho| \leq d$ with the following guarantee: If there is size-$s$ subcube partition $P \coloneqq (\rho_1, \ldots, \rho_s)$ for which $P \circ (h_{\rho_1}, \ldots, h_{\rho_s})$ has test error at most $\eps$, then the algorithm returns a size-$(k \coloneqq O(s\ln(1/\eps)))$ subcube list $L \coloneqq (\pi_1, \ldots, \pi_k)$ for which $L \circ (h_{\pi_1}, \ldots, h_{\pi_k})$ has test error at most $O(\eps \ln(1/\eps))$.
\end{lemma}
The algorithm in \Cref{lem:subcube-approx-intro} uses a greedy approach inspired by the classical approximation algorithm for set cover. We show, at each step, there is a restriction that can be added to $L$ that simultaneously covers a large portion of the remaining test examples and has small test error.

\section{Discussion and Future Work}

The data that a PAC learner receives is jointly determined by the unknown target function~$f$ and the unknown distribution $D$. However, its performance guarantees are allowed to scale with the complexity of $f$---for example, we seek $\poly(n,s)$-time algorithms for learning size-$s$ boolean circuits over $n$ variables---but usually not with the complexity of $D$. On one hand, this seems reasonable since it has long been known that the {\sl sample} complexity of learning is determined by the complexity of~$f$ (i.e.~the VC dimension of the concept class that $f$ belongs to) and suffices for {\sl all} distributions. On the other hand, given the apparent difficulty of obtaining {\sl computationally}-efficient learners even for simple concept classes, perhaps the {\sl computational} complexity of learning depends jointly on the complexity of $f$ and the complexity of~$D$ for most learning tasks. Our work contributes to the understanding of this phenomenon by giving a general method for designing efficient learners that succeed for natural notions of ``sufficiently simple" distributions, but there is much more to be done. 

For example, the focus of our work is mostly algorithmic in  nature. An avenue for future work is to study these questions from the perspective of lower bounds.  As mentioned in the introduction, the long list of lower bounds for distribution-free PAC learning shows that canonical hard problems from complexity theory and cryptography can be embedded in underlying distribution $D$, thereby negating the assumed simplicity of the target function $f$. Can we formalize and quantify this? For each concept class for which there are distribution-free lower bounds, can we show that these lower bounds in fact hold for distributions $D$ that are ``complicated" in a certain formal sense? Such an understanding will in turn suggest ways in which these lower bounds can be circumvented, and pave the way to efficient algorithms for learning these concept classes under restricted but still broad classes of distributions.

\section{Preliminaries}
\label[section]{sec:prelim}
\paragraph{Notation.}
We use \textbf{boldfont} to denote random variables and calligraphic font to denote distributions (e.g. $\bx \sim \mcD$). For any set $S$, we use $\bx \sim S$ as shorthand for $\bx$ sampled uniformly from $S$. The \textsl{support} of a distribution $\mcD$ denotes all points with nonzero probability mass. We write $[n]$ as shorthand for the set $\sset{1,2,\ldots,n}$. For any function $f:\bits^n \to \bits$ and distribution $\mcD$ over $\bits$, we use $\mcD_f$ as shorthand for the distribution of $(\bx, f(\bx))$ where $\bx \sim \mcD$.

\begin{definition}[Total variation distance]
    \label[definition]{def:TV}
    For any distributions $\mcD_1$ and $\mcD_2$ over the same finite domain $X$, the \emph{total variation distance} between $\mcD_1$ and $\mcD_2$ is defined as,
    \begin{equation*}
        \TV(\mcD_1, \mcD_2) \coloneqq \frac{1}{2}\cdot \sum_{x \in X} \abs*{\mcD_1(x) - \mcD_2(x)}
    \end{equation*}
    where $\mcD(x)$ denotes the probability mass function of $\mcD$ evaluated at $x$. 
\end{definition}

\pparagraph{Learning theory basics.}
We operate in the distribution-specific model of PAC learning. First, we define the error of a hypothesis.
\begin{definition}[Error of a hypothesis]
    For any distribution $D$ over $\bits^n \times \bits$ and hypothesis $h:\bits^n \to \bits$, we define,
    \begin{equation*}
        \error(h, D) \coloneqq \Prx_{(\bx,\by)\sim D}[h(\bx) \neq \by].
    \end{equation*}
    In a slight abuse of notation, for any $S \in \bits^n \times \bits$, we similarly define.
    \begin{equation*}
        \error(h, S) \coloneqq \Prx_{(\bx,\by)\sim S}[h(\bx) \neq \by].
    \end{equation*}
\end{definition}

\begin{definition}[Distribution-specific PAC learning]
    For any $\eps, \delta > 0$, we say a learner $A$ $(\eps,\delta)$-learns a concept class $\mcC$ over distribution $D$ if the following holds: For any $f \in \mcC$, given iid samples from $D_f$, with probability at least $1 - \delta$, $A$ returns a hypothesis $h$ with $\error(h, D_f) \leq \eps$.
\end{definition}
To simplify analysis and notation, we will drop the dependence on the failure probability and focus only on the expected accuracy.
\begin{definition}[PAC learning with expected error]
    For any $\eps > 0$, we say a learner $A$ $\eps$-learns a concept class $\mcC$ over distribution $D$ if the following holds: For any $f \in \mcC$, given iid samples from $D_f$, the hypothesis $\bh$ that $A$ returns satisfies
    \begin{equation*}
        \Ex[\error(\bh, D_f)] \leq \eps,
    \end{equation*}
    where the randomness of $\bh$ is over the randomness of the sample and $A$.
\end{definition}
It is possible to generically convert expected error learners to those that succeed with very high probability.
\begin{fact}[Boosting success probability, folklore]
    \label[fact]{fact:boost}
    For any $\eps, \delta > 0$ and efficient learner $A$ using $m$ samples, there is an efficient learner $A'$ using
    \begin{equation*}
        M \coloneqq O\paren*{m\log(1/\delta) + \frac{\log(1/\delta)}{\eps^2}}
    \end{equation*}
    with the following guarantee: If $A$ $\eps$-learns a concept class $\mcC$ over distribution $\mcD$ (in expected error), then $A'$ $(2\eps, \delta)$-learns $\mcC$ over distribution $\mcD$.
\end{fact}
The proof of \Cref{fact:boost} is simple: $A'$ runs $O(\log(1/\delta))$ independent copies of $A$. With very high probability, at least one will have low error. It then uses a test set to determine which hypothesis has the lowest error. 

With slightly more cumbersome notation, our results can be directly proven for $(\eps,\delta)$-learning. Given \Cref{fact:boost}, our focus on expected error learning is without loss of generality and simplifies notation.

\pparagraph{Restrictions.}
To formalize our lifters, we will use the concept of a \textsl{restriction}, which gives a convenient way of specifying a subcube of the domain.
\begin{definition}[Restrictions]
    \label[definition]{def:restriction}
    For the domain $\bits^n$, a \emph{restriction} is specified by a string $\rho \in \{\pm 1, \star\}^n$. An input $x \in \bits^n$ is \emph{consistent} with $\rho$, denoted $x \in \rho$, if
    \begin{equation*}
        x_i = \rho_i \text{ or }\rho_i = \star\quad\text{for all }i \in [n].
    \end{equation*}
    The \emph{depth} of $\rho$, denoted $\depth(\rho)$, is the number of $i \in [n]$ for which $\rho_i \neq \star$.
    
    For any $i \in [n]$ and $b \in \bits$, we use $x_i = b$ as shorthand for the depth-$1$ restriction $\rho$ where $\rho_i = b$ and for all $j \neq i$, $\rho_j = \star$. Similarly, for any restriction $\rho$ where $\rho_i = \star$, we use $\rho' \coloneqq \rho \cap (x_i = b)$ to denote the restriction satisfying $\rho'_i = b$ and for all $j \neq i$, $\rho'_j = \rho_j$.
\end{definition}
We extend the notation of restrictions to sets and distributions. For any $S \subseteq \bits^n$, we define,
\begin{equation*}
    S_{\rho} \coloneqq \{x\mid x \in S\text{ and } x \in \rho\}.
\end{equation*}
Similarly, for any distribution $\mcD$ over $\bits^n$, we use $\mcD_{\rho}$ to denote the distribution of $\bx \sim \mcD$ conditioned on $\bx \in \rho$.

In a slight abuse of notation, we will also extend this notation to \textsl{labeled} sets and distributions over $\bits^n \times \bits$. For any $S \subseteq \bits^n \times \bits$, we define,
\begin{equation*}
    S_{\rho} \coloneqq \sset{(x,y)\mid (x,y) \in S\text{ and } x \in \rho}.
\end{equation*}
Similarly, for a distribution $\mcD$ over $\bits^n \times \bits$, we use $\mcD_{\rho}$ to denote the distribution of $(\bx,\by) \sim \mcD$ conditioned on $\bx \in \rho$.

As an example of combining this notation, let $\mcD$ be a distribution on $\bits^n$ and $f:\bits^n \to \bits$ be some function. Then notation $(\mcD_f)_{x_i = 1}$ denotes the distribution of $(\bx, f(\bx))$ when $\bx \sim \mcD$ conditioned on $\bx_i = 1$. 

Our lifters will assume that the family of base distributions is closed under restrictions.
\begin{definition}[Closed under restrictions]
    A family of distributions, $\mcD$, is \emph{closed under restrictions} if for any $D \in \mcD$ and any restriction $\rho$, $D_{\rho} \in \mcD$.
\end{definition}
For example, the family of product distributions is closed under restrictions. Even when a family of distributions is not already closed under restrictions, this assumption is mild. Typically, it is only easier to learn over $D_{\rho}$ than $D$, so learners that succeed over some $D$ typically succeed over all restrictions of $D$.

\pparagraph{Decomposed distributions.}
The main focus of this work is to learn over a distribution $D^{\star}$ which can be decomposed into simpler distributions contained within some family of distributions, $\mcD$. We define two related ways of decomposition distributions, one based on \textsl{decision trees} and one based on \textsl{subcube covers}. We view a depth-$d$ decision tree as a collection of $2^d$ restrictions, with each restriction denoting the portion of the domain $\bits^n$ that reaches a single leaf. 
\begin{definition}[Decision trees]
    For any $0 \leq d \leq n$, a \emph{depth-$d$ decision tree $T$} over domain $\bits^n$ is a collection of $2^d$ leaf restrictions defined recursively:
    \begin{enumerate}
        \item[$\circ$] If $d = 0$, $T = \{\ell_{\mathrm{all}}\}$ just contains the single restriction $(\ell_{\mathrm{all}})_i = \star$ for all $i \in [n]$.
        \item[$\circ$] If $d > 1$, $T$ is defined by a root query $i \in [n]$ and two subtrees $T_{x_i = -1}$ and $T_{x_i = +1}$ with the property $\ell_i = \star$ for every $\ell \in T_{x_i = -1} \cup T_{x_i = 1}$. It contains the restrictions,
        \begin{equation*}
            T \coloneqq \{\ell \cap (\rho_i = -1) \mid \ell \in T_{x_i = -1}\} \cup \{\ell \cap (\rho_i = 1) \mid \ell \in T_{x_i = 1}\}.
        \end{equation*}
    \end{enumerate}
\end{definition}
\begin{definition}[Decomposition distributions via decision trees]
    \label[definition]{def:DT-decomp}
    For a family of distributions $\mcD$ over $\bits^n$, we say that a distribution $D^{\star}$ has a \emph{depth-$d$ decomposition into distributions in $\mcD$} if there is some depth-$d$ decision tree $T$ for which $(D^{\star})_{\ell} \in \mcD$ for every $\ell \in T$.
\end{definition}
We similarly decompose distributions into \emph{subcube partitions}. Note that we do not require our partition to cover the entire domain $\bits^n$. Eventually, when we apply these covers in \Cref{def:subcube-decomp}, we will require the partitions to cover the support of the distribution $D^{\star}$, which may not be all of $\bits^n$.
\begin{definition}[Subcube partitions]
    A \emph{size-$s$ depth-$d$} subcube partition is a collection of $s$ restrictions, $P \coloneqq (\rho^{(1)}, \ldots, \rho^{(s)})$, with the following properties.
    \begin{enumerate}
        \item[$\circ$] \textbf{Disjointness:} For every $i \neq j \in [s]$, there is no $x \in \bits^n$ for which $x \in \rho^{(i)}$ and $x \in \rho^{(j)}$.
        \item[$\circ$] \textbf{Codimension $d$:} For every $i \in [s]$, the number of $j \in [n]$ for which $(\rho^{(i)})_j \neq \star$ is at most $d$.
    \end{enumerate}
\end{definition}
\begin{definition}[Decomposing distributions via subcube partitions]
    \label[definition]{def:subcube-decomp}
    For a family of distributions $\mcD$ over $\bits^n$, we say that a distribution $D^{\star}$ has a \emph{depth-$d$ size-$s$ subcube decomposition into distributions in $\mcD$} if there is some size-$s$ depth-$d$ subcube partition $P$ satisfying,
    \begin{enumerate}
        \item[$\circ$] \textbf{Coverage:} For every $x$ in the support of $D^{\star}$, there is some $\rho \in P$ for which $x \in \rho$.
        \item[$\circ$] \textbf{Valid decomposition:} $(D^{\star})_{\rho} \in \mcD$ for every $\rho \in P$.
    \end{enumerate}
\end{definition}
It's straightforward to see that if $D^{\star}$ has a depth-$d$ decomposition into distributions in $\mcD$ (by decision trees), it has a depth-$d$ size-$2^d$ subcube decomposition into distributions in $\mcD$, but the converse is not true: There are depth-$O(\log n)$ size-$O(n)$ subcube decompositions that require a depth-$\tilde{\Omega}(n)$ decomposition by decision trees. Therefore, our lifter for subcube covers can handle a much more expressive class of distributions.\footnote{This separation is achieved by setting $D^{\star}$ to be the uniform distribution over accepting inputs of the $\textsc{Tribes}$ function and $\mcD$ to consist of the uniform distribution and all its restrictions. See \citep{ODBook} for the definition of and standard properties of $\textsc{Tribes}$ that imply the desired separation.}

\ifnum\coltformat=0
\section{Our lifter for decision tree decompositions and its analysis}
\fi
\ifnum\coltformat=1
\section{Lifting learners over decision tree decompositions}
\fi
\newcommand{\out}{\mathrm{out}}

\begin{theorem}[Our lifter for decision tree decompositions]
    \label{thm:lift-DT}
Let $\mcC$ be a concept class and $\mathcal{D}$ be a class of distributions that is closed under restrictions. Let $A$ be a base learner that uses $t$ time and $m$ samples to learn $\mcC$ to expected error $\leq \eps$ under any distribution $D \in \mcD$. Then, for
\begin{equation*}
    m' \coloneqq O\paren*{\frac{2^d m}{\eps} + \frac{2^d \log n}{\eps^2}} \quad\quad\text{and} \quad\quad t' \coloneqq \frac{n^{O(d)} m't \log t}{\eps^2}.
\end{equation*}
there is a $t'$-time $m'$-sample algorithm that learns $\mcC$ to expected error $\eps$ under any distribution $D^{\star}$ that has a depth-$d$ decomposition into distributions in $\mcD$.
\end{theorem}

The definition of a depth-$d$ decomposition, given informally in the introduction, is detailed in \Cref{def:DT-decomp}. The algorithm used in the proof of \Cref{thm:lift-DT} is the algorithm $\TreeLearn_{A}$, described in~\Cref{fig:dtlearn}. We further show the noise tolerance of this algorithm in \Cref{thm:lift-robustly}.

In that algorithm, we use notation to compose a decision tree with hypotheses:
\begin{definition}[Decision tree composed hypotheses]
    For any decision tree $T$ and mapping $\mcH$ from restrictions to hypotheses, we use $T \circ \mcH$ to denote the function that, on input $x$, finds the unique leaf $\rho \in T$ for which $x \in \rho$ and outputs $\mcH(\rho)(x)$.
\end{definition}
We analyze $\TreeLearn_{A}$'s time complexity in~\Cref{sec:time-efficiency}, and then prove its correctness and analyze its sample complexity~\Cref{sec:correctness}. We analyze its noise tolerance in~\Cref{sec:noise tolerance}.

\begin{figure}[h] 
  \captionsetup{width=.9\linewidth}
    
    \begin{tcolorbox}[colback = white,arc=1mm, boxrule=0.25mm]
    \vspace{2pt} 
    $\TreeLearn_{A}(S_{\train}, S_{\test}, d):$\vspace{6pt}
    
    \textbf{Input:} A training set $S_{\train}$, test set $S_{\test}$, and depth $d$. \vspace{6pt}

    \textbf{Output:} A hypothesis $H$.\vspace{6pt}

    \textbf{Training phase:} For every depth-$d$ restriction $\rho \in \sset{\pm 1, \star}^n$, set $h_{\rho} \leftarrow A((S_{\train})_{\rho})$. Let $\mcH$ denote the mapping from $\rho$ to $h_{\rho}$.\vspace{6pt}

    \textbf{Test phase:} Let $T \leftarrow \FindTree(S_{\test}, d, \mcH, \rho_{\mathrm{all}})$, where $(\rho_{\mathrm{all}})_i = \star$ for all $i \in [n]$, be the decision tree minimizing $\error(T \circ \mcH, S_{\test})$. Return $H \coloneqq T \circ \mcH$.\vspace{2pt}

    \end{tcolorbox}

     \begin{tcolorbox}[colback = white,arc=1mm, boxrule=0.25mm]
    \vspace{2pt} 
    $\FindTree_{A}(S_{\test}, d, \mcH, \rho):$\vspace{6pt}
    
    \textbf{Input:} A test set $S_{\test}$, depth $d$, mapping from restrictions to hypotheses, $\mcH$, and current restriction $\rho$.\vspace{6pt}

    \textbf{Output:} A depth-$d$ tree $T$ that minimizes $\error(T \circ \mcH, (S_{\test})_{\rho})$ .

    \begin{enumerate}
        \item If $d = 0$, return $T \coloneqq \sset{\rho_{\mathrm{all}}}$.
        \item Otherwise, for each $i \in \rho^{-1}(\star)$,
        \begin{itemize}
            \item[$\circ$] \textbf{Recursive step:} Let $T_{x_i =1} \leftarrow \FindTree_{A}(S_{\test}, d-1, \mcH, \rho \cap (x_i = 1))$ and  $T_{x_i =-1} \leftarrow \FindTree_{A}(S_{\test}, d-1, \mcH, \rho \cap (x_i = -1))$.
            \item[$\circ$] \textbf{Combine trees:} Let $T^{(i)}$ be the tree with $x_i$ at the root and $T_{x_i =-1}$ and $T_{x_i =1}$ as its left and subtrees respectively.
        \end{itemize}
        \item Return the $T^{(i)}$ the minimizes $\error(T^{(i)}, (S_{\test})_{\rho})$ among $i \in \rho^{-1}(\star)$.
    \end{enumerate}
    \end{tcolorbox}
\caption{The lifting algorithm with the guarantees in \Cref{thm:lift-DT}.}
\label{fig:dtlearn}
\end{figure}



\ifnum\coltformat=0
\subsection{Time-efficiency of $\TreeLearn_A$}\label{sec:time-efficiency}

\begin{lemma}[Time Efficiency]\label{lemma:dtlearn-efficient}
Let $A$ be a time-$t$ learner, $|S_{\train}| = m_{\train}$ and $|S_{\test}| = m_{\test}$. $\TreeLearn$ terminates in time,
\begin{equation*}
    t' \coloneqq O\paren*{n^{O(d)} \cdot (m_{\train} + m_{\test}) \cdot t \log t}.
\end{equation*}
\end{lemma}
\begin{proof}
    In the training phase, $\TreeLearn$ one call to $A$ for each of the $n^{O(d)}$ depth-$d$ restrictions $\rho$. For each such restriction, $\TreeLearn$ needs to filter the size $m_{\train}$ dataset to just those samples that are consistent with $\rho$ and make a call to $A$. The total time for this is easily upper bounded by $n^{O(d)} \cdot (m_{\train} t)$.

    In the test phase, we observe that $\FindTree$ makes a total of $1 + 2n + (2n)^2 +\cdots + (2n)^{d} = n^{O(d)}$ recursive calls. The most expensive computation each call makes is to determine the accuracy of $n$ hypotheses. Since $A$ runs in time-$t$, the largest hypothesis it can output is a size-$t$ circuit, which can be evaluated in time $O(t \log t)$ \cite{AB09}. Therefore, this step takes at most $n^{O(d)} \cdot m_{\test} \cdot O(t \log t)$ time.
\end{proof}

    
    

    

\subsection{Correctness of $\TreeLearn$}\label{sec:correctness}

\begin{lemma}[Correctness]\label{lemma:dtlearn-correct} 
Let $\mcC$ be a concept class, $\mathcal{D}$ be a class of distributions that is closed under restrictions, $D^{\star}$ a distribution that has a depth-$d$ decomposition into distributions in $\mcD$, and $A$ be a base learner using $m$ samples to learns $\mcC$ to expected error $\leq \eps$ under any distribution $D \in \mcD$. Then for any $f \in \mcC$ and,
 \begin{equation*}
     m_{\train} \coloneqq O\paren*{\frac{2^d m}{\eps}} \quad\quad\text{and}\quad\quad m_{\test} \coloneqq O\paren*{\frac{2^d \ln n}{\eps^2}},
 \end{equation*}
if $\bS_{\mathrm{train}}$ and $\bS_{\mathrm{test}}$ are samples consisting $m_{\train}$ and $m_{\test}$ respectively i.i.d.~examples from $D^*_f$, $\TreeLearn_A(\bS_{\train}, \bS_{\test}, d)$ returns a hypothesis $\bH$ satisfying $\Ex[\error(\bH, D^{\star}_f)] \leq 4\eps$.
\end{lemma}

Our proof has three parts. First, we show that $\FindTree_A$ returns the hypothesis with the lowest error on $S_{\test}$ (\Cref{claim:minimize-emp-loss}). Next, we show, in expectation over the randomness of $\bS_{\train}$ and the learner $A$, there is a tree $T$ for which $\Ex[\error(T \circ \mcH, D^{\star}_f)] \leq \eps$. With these two results, the proof would be complete if $\error(T \circ \mcH, D^{\star}_f) =\error(T \circ \mcH, S_{\test})$. Indeed, using \Cref{claim:emploss-closs-trueloss}, we show that these two values cannot differ much (in expectation) for any depth-$d$ tree $T$.
\begin{claim}[$\FindTree$ minimizes empirical loss.]\label{claim:minimize-emp-loss}
    For any test set $S_{\test}$, depth $d$, mapping of restrictions from hypothesis $\mcH$, and restriction $\rho$, $\FindTree_A(S_{\test}, d, \mcH, \rho)$ returns a tree $T$ satisfying,
    \begin{equation}
        \label{eq:best-tree}
        \error(T \circ \mcH, (S_{\test})_{\rho}) = \min_{T'\text{ is a depth-$d$ decision tree}}\sset*{\error(T' \circ \mcH, (S_{\test})_{\rho})}
    \end{equation}
\end{claim}
\begin{proof}
    By induction on the depth $d$. If $d = 0$, the desired claim holds because there is only one possible tree structure of depth $0$.

    For the inductive step, consider any $d > 1$. For any tree $T$ with root $x_i$ and subtrees $T_{-1}$ and $T_1$, we observe its error can be easily decomposed: Using $\mathrm{raw\_error}(h, S) \coloneqq |S| \cdot \error(h)$ to denote the number of points in $S$ that $h$ fails to classify,
    \begin{equation*}
        \mathrm{raw\_error}(T \circ \mcH, (S_{\test})_{\rho}) = \mathrm{raw\_error}(T_1 \circ \mcH, (S_{\test})_{\rho} \cap (x_{i} = 1)) + \mathrm{raw\_error}(T_{-1} \circ \mcH, (S_{\test})_{\rho} \cap (x_{i} = -1)).
    \end{equation*}
    Next, let $T'$ be a tree minimizing the error in \Cref{eq:best-tree} and $x_{i^{\star}}$ be its root. Using the notation in $\FindTree$, we claim that $\error(T^{(i^\star)}, (S_{\test})_{\rho}) \leq \error(T', (S_{\test})_{\rho})$, which implies the desired result since $\FindTree$ returns the tree with minimum error. This follows from the inductive hypothesis: Each of the subtrees ($T_{x_{i^\star} =-1}$ and $T_{x_{i^\star} =1}$) will have error at least as good as the corresponding subtrees of $T'$. Since the error of a tree decomposes in terms of its subtrees' error, this implies $T^{(i^\star)}$ has error at least as good as $T'$. 
\end{proof}

\begin{claim}[The existence of a tree with low error in expectation]\label{claim:T-with-lowloss}
    In the context of \Cref{lemma:dtlearn-correct}, for any,
    \begin{equation*}
        m_{\train} \geq \frac{2^d}{\eps} \cdot \max\paren*{2m, 8},
    \end{equation*}
    and any $f \in \mcC$, let $\bS_{\train}$ consist of $m_{\train}$ i.i.d. examples from $D^{\star}_f$. Then, let $\bmcH$ denote the set of hypotheses ($\mcH$ in \Cref{fig:dtlearn}) found by $\TreeLearn_A$, where the randomness of $\bmcH$ is the randomness of $\bS_{\train}$ and the base learner $A$.
    \begin{equation*}
        \Ex_{\bmcH}\bracket*{\min_{\text{depth-$d$ tree $T$}} \error(T \circ \bmcH, D^{\star}_f)} \leq 2\eps.
    \end{equation*}
\end{claim}
Our proof of \Cref{claim:T-with-lowloss} will use the following standard Chernoff bound.
\begin{fact}[Chernoff bound]
    \label{fact:chernoff-lower}
    Let $\bz$ be the sum of independent random variables each supported on $\zo$ and $\mu$ its mean. Then,
    \begin{equation*}
        \Pr[\bz \leq \mu/2] \leq e^{-\mu/8}.
    \end{equation*}
\end{fact}

\begin{proof}[Proof of \Cref{claim:T-with-lowloss}]
Let $T^{\star}$ be the depth-$d$ tree decomposing $D^{\star}$ in $\mcD$ guaranteed to exist by \Cref{def:DT-decomp}. We will show that
\begin{equation*}
    \Ex_{\bmcH}\bracket*{ \error(T^{\star} \circ \bmcH, D^{\star}_f)} \leq 2\eps,
\end{equation*}
which implies the desired bound. The key property of this $T^{\star}$ is: for every $\ell \in T^{\star}$, conditioned on there being at least $m$ points in $(\bS_{\train})_{\rho}$, the expected error of $\mcH(\ell)$ on $(D^{\star}_f)_{\ell}$ is at most $\eps$. This follows from $(D^{\star})_\ell$ being within the family of distributions $\mcD$ and the guarantees of the learner $A$ in \Cref{thm:lift-DT}.

Therefore, the main technical step in this proof is to argue that all leaves $\ell \in T$ for which $\abs*{(\bS_{\train})_{\ell}} \leq m$ collectively have small mass (under $D^{\star}$) in expectation (over the randomness of $\bS_{\train}$). To formalize this analysis, let $p(\ell)$ be shorthand for
\begin{equation*}
    p(\ell) \coloneqq \Prx_{\bx \sim D^{\star}}[\bx \in \ell].
\end{equation*}
We will show that for any leaf $\ell$,
\begin{equation}
    \label{eq:weight-inequality}
    p(\ell) \cdot \Pr\bracket*{\abs*{(\bS_{\train})_{\rho}} \leq m} \leq \frac{\eps}{2^d}. 
\end{equation}
If $p(\ell) \leq \frac{\eps}{2^d}$, clearly \Cref{eq:weight-inequality}. Otherwise by the Chernoff bound \Cref{fact:chernoff-lower} and the requirement that $m_{\train} \geq \frac{2m 2^d}{\eps}$,
\begin{equation*}
    \Pr\bracket*{\abs*{(\bS_{\train})_{\rho}} \leq m}  \leq e^{-m_{\train} \cdot p(\ell)/8}.
\end{equation*}
Here, we use that the function $p \mapsto pe^{-p m_{\train}/8}$ is decreasing whenever $p\geq 8/m_{\train}$. By the requirement that $m_{\train} \geq \frac{8 \cdot 2^d}{\eps}$, we therefore have the quantity $ p(\ell) \cdot \Pr\bracket*{\abs*{(\bS_{\train})_{\rho}} \leq m}$ is maximized when $p(\ell) =  \frac{\eps}{2^d}$, in which case \Cref{eq:weight-inequality} still holds.

Finally, we bound the expected error of $T^{\star} \circ \bmcH$:
\begin{align*}
    &\Ex_{\bmcH}\bracket*{ \error(T^{\star} \circ \bmcH, D^{\star}_f)} \\
    &\quad\quad= \sum_{\ell \in T^{\star}} \Ex_{\bmcH}\bracket*{ \error(\bmcH(\ell), (D^{\star}_f)_\ell)} \tag{Linearity of expectation}\\
    &\quad\quad\leq \sum_{\ell \in T^{\star}} p(\ell) \cdot \paren*{\Pr\bracket*{\abs*{(\bS_{\train})_{\rho}} \leq m} + \eps} \tag{$A$ has expected error $\eps$ if given $\geq m$ samples}\\
    &\quad\quad\leq \sum_{\ell \in T^{\star}} \frac{\eps}{2^d} + \eps \tag{\Cref{eq:weight-inequality}} \\
    &\quad\quad= 2\eps \tag{$T^{\star}$ has $2^d$ leaves}
\end{align*}
\end{proof}

Next we use a standard argument to connect empirical error to true error.
\begin{claim}[Empirical error is close to true error.]\label{claim:emploss-closs-trueloss}
    Let $\mcG$ be a set of functions and $\bS_{\test}$ consist of $m$ i.i.d. points from $D^{\star}_f$. Then,
    \begin{equation*}
        \Ex\bracket*{\sup_{g \in \mcG}\abs*{\error(g, D^{\star}_f) - \error(g, \bS_{\test}) }} \leq \sqrt{\frac{\ln (2|\mcG|)}{2m}}
    \end{equation*}
\end{claim}
The proof of \Cref{claim:emploss-closs-trueloss} uses two standard facts about sub-Gaussian random variables. We refer the reader to \cite{Ver18Book} for a more thorough treatment of sub-Gaussian random variables.
\begin{definition}[Sub-Gaussian random variables]
    \label{def:sub-gaussian}
    A random variable $\bz$ is said to be \emph{sub-Gaussian} with variance-proxy $\sigma^2$ if, for all $\lambda \in \R$,
    \begin{equation*}
        \Ex[e^{\lambda \bz}] \leq \exp\paren*{\tfrac{\lambda^2 \sigma^2}{2}}.
    \end{equation*}
\end{definition}
\begin{fact}[Mean of Bernoulli random variables is sub-Gaussian]
    \label{fact:Ber-sub-Gaussian}
    Let $\bz_1, \ldots, \bz_m$ each be independent random variables supported on $\zo$ and $\bZ \coloneqq \frac{\bz_1 + \cdots + \bz_m}{m}$ is sub-Gaussian with variance-proxy $\frac{1}{4m}$.
\end{fact}
\begin{fact}[Expected maximum of sub-Gaussian random variables]
    \label{fact:max-sub-gaussian}
    Let $\bz_1, \ldots, \bz_n$ be (not necessarily independent) sub-Gaussian random variables each with variance proxy $\sigma^2$. Then,
    \begin{equation*}
        \Ex\bracket*{\max_{i \in [n]} \abs*{\bz_i}} \leq \sqrt{2 \sigma^2 \ln(2n) }.
    \end{equation*}
\end{fact}

\begin{proof}[Proof of \Cref{claim:emploss-closs-trueloss}]
    By \Cref{fact:Ber-sub-Gaussian}, for any fixed $g\in \mcG$, the quantity $\error(g, D^{\star}_f) - \error(g, \bS_{\test})$ is sub-Gaussian with variance proxy $\frac{1}{4m}$. The desired bound follows from \Cref{fact:max-sub-gaussian}.
\end{proof}

\paragraph{Proof of~\Cref{lemma:dtlearn-correct}.}
We have now assembled all the ingredients needed to prove the main result of this subsection. 
\begin{proof}
    An easy counting argument gives that there are at most $n^{2^d}$ decision trees of depth $d$. Therefore, by \Cref{claim:emploss-closs-trueloss}, for any fixed $\mcH$ and $m_{\test} = O\paren*{\frac{2^d \ln n}{\eps^2}}$,
    \begin{equation*}
        \Ex_{\bS_{\test}}\bracket*{\sup_{\text{depth-$d$ tree T}}\abs*{\error(T \circ \mcH, D^{\star}_f) - \error(T \circ \mcH, \bS_{\test}) }} \leq \eps.
    \end{equation*}
    If we view $\bmcH$ as a random variable (with randomness based on the randomness of $\bS_{\train}$), then have that $\bmcH$ and $\bS_{\test}$. Therefore, conditioning on any fixed value of $\bmcH$, $\bS_{\test}$ still has the same distribution, and so we still have that,
    \begin{equation}
        \label{eq:tree-close}
        \Ex_{\bmcH, \bS_{\test}}\bracket*{\sup_{\text{depth-$d$ tree T}}\abs*{\error(T \circ \mcH, D^{\star}_f) - \error(T \circ \mcH, \bS_{\test}) }} \leq \eps.        
    \end{equation}
    By \Cref{claim:T-with-lowloss}, over the randomness of $\bmcH$, there is a tree $T^\star$ with expected error on $D^{\star}_f$ at most $2\eps$. By \Cref{eq:tree-close}, over the randomness of $\bmcH$ and $\bS_{\test}$, there is a tree with expected error on $\bS_{\test}$ at most $3\eps$. Since $\TreeLearn$ outputs the tree with minimal test error (\Cref{claim:minimize-emp-loss}), the expected error on $\bS_{\test}$ of the hypothesis that $\FindTree$ outputs is at most $3\eps$. One more application of \Cref{eq:tree-close} gives that the expected error on $D^{\star}_f$ of the hypothesis outputted is at most $4\eps$.
\end{proof}

\fi

\subsection{Noise tolerance}
\label{sec:noise tolerance}
Our lifter generically upgrades \textsl{robust} learners on the base distributions $\mcD$ to a \textsl{robust} learner on $D^{\star}$. We demonstrate this with the following notion of robustness, though similar techniques would apply to other notions.
\begin{definition}[Robust learner]
    \label{def:robust-learner}
    We say a learner $A$ $c$-robustly $\eps$-learns a concept class $\mcC$ over distribution $D$ if the following is true: For any $f \in \mcC$, given iid samples from some $D_f'$, $A$ returns a hypothesis $\bH$ satisfying
    \begin{equation*}
       \Ex[\error(\bH, f)] \leq c \cdot \TV(D_f, D_f') + \eps.
    \end{equation*}
\end{definition}
Put differently, if an adversary corrupts $\eta$-fraction of the labeled distribution, a robust learner should return a hypothesis with $O(\eta)$ error in expectation. Recent work \citep{BV24} showed that this notion of robustness is equivalent to nasty noise \citep{BEK02} and strong contamination \citep{DK23} in the sense that a learner satisfying \Cref{def:robust-learner} can be black-box transformed into a learner that is robust to nasty noise and strong contamination with only a polynomial increase in runtime and sample size.

\begin{theorem}[Lifting robust learners]
    \label{thm:lift-robustly}
    Let $\mcC$ be a concept class and $\mcD = \sset{D_1, D_2, \ldots}$ be a class of distributions that is closed under restrictions. If there is a learning algorithm $A$ running in time $t$ and using $m$ samples that $c$-robustly $\eps$-learns $\mcC$ over every $D \in \mcD$, for
    \begin{equation*}
        T \coloneqq \poly(n^{d}, 1/\eps, t) \quad\quad\text{and}\quad\quad M \coloneqq \poly(2^d, 1/\eps, m, \log n),
    \end{equation*}
    then there is an algorithm $A'$ running in time $T$ and using $M$ samples that $(2c+2)$-robustly $O(\eps)$-learns $\mcC$ over any distribution $D^{\star}$ that has a depth-$d$ decomposition into distribution in $\mcD$.
\end{theorem}
The key property we need about total variation distance we need is the following.
\begin{fact}[Total variation distance distributes over the leaves of a tree, Lemma B.4 of \cite{BLMT-boosting}]
    \label[fact]{fact:TV-distribute}
    For any distributions $\mcD, \mcD'$ with $\TV(\mcD, \mcD') \leq \eta$ and decision tree $T$,
    \begin{equation*}
        \sum_{\ell \in T}\Prx_{\bx \sim \mcD'}[\bx\text{ reaches }\ell\text{ in }T] \cdot \TV(\mcD_{\ell}, (\mcD')_{\ell}) \leq 2\eta.
    \end{equation*}
\end{fact}
Since the proof of \Cref{thm:lift-robustly} is essentially the same as \Cref{thm:lift-DT}, we just sketch the structure of the argument and (minor) differences. We use the same algorithm, $\TreeLearn$ from \Cref{fig:dtlearn}. Recall there are three key steps to proving the accuracy of that algorithm:
\begin{enumerate}
    \item \Cref{claim:minimize-emp-loss}, that this algorithm finds the tree with the minimum error on the test set. This step remains unchanged in the robust setting.
    \item \Cref{claim:T-with-lowloss}, that there is a tree with low true error. This step is extremely similar, except we now need to apply \Cref{fact:TV-distribute}. If the adversary decided to corrupt $\eta$ fraction of the labeled distribution, by \Cref{fact:TV-distribute}, the average leaf will have at most $2\eta$ corruptions. Hence, using the same logic as in the proof of \Cref{claim:T-with-lowloss}, the expected error of the true tree will be at most $2c\eta + 2\eps$.
    \item \Cref{claim:emploss-closs-trueloss}, that empirical error is close to the true error. This step is also essentially the same, except that the expected difference between the emprical error and true error will be upper bounded by $(\eta + \eps)$ rather than $\eps$. This is because the test set will give an $\eps$-accurate estimate of the hypotheses error on the \textsl{noisy distribution} $(\mcD_c)'$ (as defined in \Cref{def:robust-learner}), and so these estimates may only be $(\eps + \TV(\mcD_f, (\mcD_f)')$ accurate.
\end{enumerate}
Combining these, $\textsc{Learn}_A$ is guaranteed to return a hypothesis with empirical error $2c\eta + \eta = (2c+1)\eta$, which translates to at most $(2c + 1)\eta + \eta = (2c+2)\eta$ error on the true distribution (without noise).
\section{Lifting learners over subcube partition distributions}

\ifnum\coltformat=0\subsection{Subcube lists and the lifting algorithm}
\fi
In this section, we present a learner that works when the distribution has a small \emph{subcube partition}
into instances of the base distribution family (see \Cref{def:subcube-decomp}).
We define a \emph{subcube list hypothesis}, since the hypothesis we return will
have a set of restrictions that are not necessarily disjoint.
\begin{definition}[Subcube list hypothesis]
For an ordered list of restrictions $L$ and a mapping $\mcH$ from restrictions to hypotheses,
the subcube list hypothesis $L \circ \mcH$ is defined as follows:
\[(L \circ \mcH)(x) = \begin{cases} \mcH(\rho)(x)\text{ for the first subcube $\rho \in L$ consistent with $x$}~&|~x\text{ is consistent with some }\rho \in L\\
		\bot~&|~\text{ otherwise.}\end{cases}\]
\end{definition}

The main result of this section is the following:
\begin{theorem}[Lifting learners using subcube lists]
\label{thm:subcube-partition}
Let $\mcC$ be a concept class and $\mcD$ be a class of distributions closed under restrictions.
Let $A$ be a base learner that learns $\mcC$ over distributions in $\mcD$ to expected error $\le \frac{\eps}{\log(1/\eps)}$, with sample complexity $m$ and time complexity $t$.
Then there is an algorithm that learns $\mcC$ over any distribution that admits
a size-$s$ subcube partition of codimension $d$ into distributions in $\mcD$, with expected error $O(\eps)$.
The algorithm has sample complexity 
\[m' \coloneqq O\paren*{\frac{ms \log(1/\eps)}{\eps} + \frac{sd\log (n) \log^3(1/\eps)}{\eps^2}}\] and time complexity 
\[t' \coloneqq n^{O(d)} \cdot \paren*{t \log t + \frac{\log^4(1/\eps)}{\eps^2}}.\]
\end{theorem}

We will analyze the sample complexity and correctness in \Cref{lem:partitionlearn-correct}, and the time complexity in \Cref{lem:time-partition}
\ifnum\coltformat=1
both in \Cref{appendix:subcube}
\fi
.
The algorithm is described as follows. 
\begin{figure}[h] 
  \captionsetup{width=.9\linewidth}
\begin{tcolorbox}[colback = white,arc=1mm, boxrule=0.25mm]
$\textsc{PartitionLearn}_A(d, s, \eps, S_{\train}, S_{\test})$:
\vspace{6pt} 

\textbf{Input:} codimension bound $d$, size bound $s$, error bound $\eps$, failure probability bound $\delta$, sample sets $S_{\train}, S_{\test}$ drawn from
the unknown distribution $D^\star$.
\vspace{6pt}

\textbf{Output:} hypothesis $h$.
\vspace{6pt}

\textbf{Training phase:} For every restriction $\rho \in \sset{\pm 1, \star}^n$ with depth at most $d$, set $h_{\rho} \leftarrow A((S_{\train})_\rho)$. Let $\mcH$ denote the mapping from $\rho$ to $h_{\rho}$.\vspace{6pt}

\textbf{Test phase:} Let $L \leftarrow$ \FindSubcube$(S_{\test}, \mcH, \eps, s)$, as defined in \Cref{fig:findSubcube}. Return $H \coloneqq L \circ \mcH$.
\end{tcolorbox}
\caption{An algorithm with the guarantees described in \Cref{thm:subcube-partition}}
\label{fig:partition-learn}
\end{figure} 

\ifnum\coltformat=0
\subsection{Correctness of PartitionLearn}
\begin{lemma}[Correctness]
\label{lem:partitionlearn-correct}
Let $\mcC$ be a concept class and $\mcD$ be a class of distributions closed under restrictions,
$D^\star$ a distribution that has a depth-$d$, size-$s$ subcube partition into $\mcD$.
Let $A$ be a base learner that learns $\mcC$ over distributions in $\mcD$ to expected error $\le \frac{\eps}{\log(1/\eps)}$, with sample complexity $m$.
Then for any $f \in \mcC$, with samples $\bS_\train$ and $\bS_\test$ drawn i.i.d. from $D^\star_f$ of sizes
\[m_\train \coloneqq O\paren*{\frac{s \log(1/\eps)}{\eps}  \cdot 
m} \quad \text{and} \quad m_\test \coloneqq O\paren*{\frac{sd \log(n) \log^3(1/\eps)}{\eps^2}},\]
{\sc PartitionLearn}$_A(d, s, \eps, \bS_\train, \bS_\test)$ returns a hypothesis $\bH$ satisfying $\E[\error(\bH, D^\star_f)] \le O(\eps)$.
    
\end{lemma}

First, we will argue that there exists a subcube partition $L$ such that in expectation over $\bS_{\train}$,
after the training phase
the error of 
$L \circ \mcH$ is sufficiently small.
This immediately follows from the proof of \Cref{claim:T-with-lowloss}, as \Cref{claim:T-with-lowloss} is in fact a special case of this statement.
\begin{corollary}[Existence of a subcube partition with low expected error]
\label{cor:exists-good-partition}
In the context of \Cref{lem:partitionlearn-correct}, for any
    \begin{equation*}
        m_{\train} \geq s \cdot \frac{\log(1/\eps)}{\eps} \cdot \max\paren*{2m, 8},
    \end{equation*}
    and any $f \in \mcC$, let $\bS_{\train}$ consist of $m_{\train}$ i.i.d. examples from $D^{\star}_f$. Then, let $\bmcH$ denote the set of hypotheses ($\mcH$ in \Cref{fig:partition-learn}) found by {\sc PartitionLearn}$_A$, where the randomness of $\bmcH$ is the randomness of $\bS_{\train}$ and the base learner $A$. Let $L^\star$ be the ground-truth subcube partition of $D^\star$ (ordered arbitrarily). We then have
    \begin{equation*}
        \Ex_{\bmcH}\bracket*{\error(L^\star \circ \bmcH, D^{\star}_f)} \leq \frac{2\eps}{\log(1/\eps)}.
    \end{equation*}
\end{corollary}
\begin{proof}
We may observe that the proof of \Cref{claim:T-with-lowloss} does not require the decomposition of $D^\star$ to have treelike structure;
rather, the assumption of $D^\star$ having a depth-$d$ tree decomposition
is used only to prove that $D^\star$ has a subcube partition of size $2^d$.
This corollary follows by substituting $s$ for $2^d$ and $\frac{\eps}{\log(1/\eps)}$ for $\eps$.
\end{proof}

\subsubsection{Correctness of greedy search}
Now we argue the analogue of \Cref{claim:minimize-emp-loss}, showing that {\sc FindSubcubeList} finds a list $L$ such that $\error(L \circ \mcH, S_\test) \le O(\eps)$. 
Here we no longer argue about sample complexity; instead we fix $S_\test$ and the mapping $\mcH$.
Since here we cannot exactly 
minimize test loss by backtracking search as in the decision tree setting, we approximately minimize it by an approach inspired by greedy approximate set-cover instead.


\begin{lemma}
    \label{lem:subcube-approx}
    The algorithm $\FindSubcube$ in \Cref{fig:findSubcube} has the following guarantee: If there is some disjoint list of $s$ restrictions with depth-$(\leq d)$ restrictions $P$ for which $P \circ \mcH$ has error at most $\eps$ on $S_{\test}$, then $\FindSubcube(S_{\test}, \mcH, \epsa ,s)$ outputs a (not necessarily disjoint) list $L$ satisfying $\error(L \circ \mcH, S_\test) \leq O(\eps \log(1/\epsa) + \epsa)$.
\end{lemma}

\begin{figure}[h] 
  \captionsetup{width=.9\linewidth}
    
    \begin{tcolorbox}[colback = white,arc=1mm, boxrule=0.25mm]
    \vspace{2pt} 
    \FindSubcube$(S_{\test}, \mcH, \epsa, s):$\vspace{6pt}
    
    \textbf{Input:} A test set $S_{\test}$, mapping $\mcH$ from restrictions to hypotheses, and parameters $\epsa, s$. \vspace{6pt}

    \textbf{Output:} A hypothesis $H$.\vspace{6pt}

    \textbf{Initialize:} Set $S_{\remaining} \leftarrow S_{\test}$ and $L$ to an empty list. \vspace{6pt}

    \textbf{Greedily grow list:} While $\frac{|S_{\remaining}|}{|S|} > \epsa$ and at most $\ceil{2s\ln(1/\epsa)}$ iterations have run:
    \begin{enumerate}
            \item Among all depth-$(\leq d)$ restrictions $\rho$ satisfying that a large fraction of the remaining test set is consistent with $\rho$,
            \begin{equation*}
                \frac{\abs*{(S_{\remaining})_{\rho}}}{\abs*{S_{\remaining}}} \geq \frac{1}{2s},
            \end{equation*}
            set $\rho^\star$ to minimize $\error(\mcH(\rho), (S_{\remaining})_{\rho})$.
                
            
            \item Update parameters: Append $\rho^\star$ to $L$ and set $S_{\remaining} \leftarrow S_{\remaining} \setminus (S_{\remaining})_{\rho^\star}.$
    \end{enumerate}
            
  \textbf{Return} $L$.
    \end{tcolorbox}
\caption{An algorithm with the guarantees described in \Cref{lem:subcube-approx}}
\label{fig:findSubcube}
\end{figure} 

The first step of this proof is to show that in every step, there exists some restriction that simultaneously has low error over the remaining test sample and covers
a large enough fraction of it.

\begin{claim}[A greedy piece exists]
    \label{claim:greedy-exists}
    Let $S \in (\bits^n \times \bits)^m$ be a labeled test set and $P \coloneqq \{\rho_1, \ldots, \rho_s\}$ be a set of disjoint restrictions for which $\error(P \circ \mcH, S) \leq \eps$. Then, for any $S_{\remaining} \subseteq S$, there exists some $i \in [s]$ for which,
    \begin{enumerate}
        \item $\rho_i$ contains a large fraction of $S_{\remaining}$,
        \begin{equation}    
            \label{eq:large-piece}
            \Prx_{(\bx,\by) \sim S_{\remaining}}[\bx \in \rho_i] \geq \frac{1}{2s}.
        \end{equation}
        \item $\mcH(\rho_i)$ has little error,
        \begin{equation}
            \label{eq:good-piece}
            \error(\mcH(\rho_i), (S_{\remaining})_{\rho_i}) \leq 2\eps \cdot \frac{|S|}{|S_{\remaining}|}
        \end{equation}
    \end{enumerate}
\end{claim}
\begin{proof}
    Let $p_i \coloneqq \frac{|(S_{\remaining})_{\rho}|}{|(S_{\remaining})|}$ and $p_{\mathrm{total}} = \sum_{i = 1}^m p_i$. By the definition of error, we have that
    \begin{align*}
         (1 - p_{\mathrm{total}}) + \sum_{i \in [m]} p_i\cdot \error(\mcH(\rho_i), (S_{\remaining})_{\rho}) &= \error(P \circ \mcH, S_{\remaining}) \\
         & \leq \frac{|S|}{|S_{\remaining}|} \cdot \error(P \circ \mcH, S)\\
         & \leq \frac{|S|}{|S_{\remaining}|} \cdot \eps.
    \end{align*}
    Let $G \subseteq [s]$ consist of all $i \in [s]$ satisfying \Cref{eq:good-piece}. Then, we must have that $\sum_{i \in G} p_i \geq 1/2$. Otherwise, the above equation would be violated (since $1/2$ of $S_{\remaining}$ would have loss greater than $2\eps \cdot \frac{|S|}{|S_{\remaining}|}$). Finally, we observe that since $|G| \leq s$, there is some $i \in G$ for which $p_i \geq \frac{1}{2s}$. This $i$ meets both of the desired criteria.
\end{proof}

\begin{proof}[Proof of \Cref{lem:subcube-approx}]
    By \Cref{claim:greedy-exists}, at each step, $\FindSubcube$ finds a restriction $\rho$ satisfying,
    \begin{equation*}
         \frac{\abs*{(S_{\remaining})_{\rho}}}{\abs*{S_{\remaining}}} \geq \frac{1}{2s} \quad\quad\text{and}\quad\quad \error(\mcH(\rho), (S_{\remaining})_{\rho}) \leq 2\eps \cdot \frac{|S_{\test}|}{|S_{\remaining}|}.
    \end{equation*}
     First, we show that when the loop in $\FindSubcube$ terminates, $\frac{|S_{\remaining}|}{|S|} \leq \epsa$. Observe that at each step, the size of $S_{\remaining}$ shrinks by a multiplicative factor of $(1 - 1/(2s))$ or better. Therefore, after $ 2s \ln (1/\epsa)$ iterations, $\frac{|S_{\remaining}|}{|S|} \leq \epsa$.

    Next, we bound the test loss of the hypothesis, $L \circ \mcH$, for the list that $\FindSubcube$ outputs. For this, it will be helpful to introduce some notation:
    \begin{enumerate}
        \item Let $r_i$ be ratio $\frac{|S_{\remaining}|}{|S|}$ at the start of the $i^{\text{th}}$ iteration, so that $r_1 = 1$, $r_{k} > \epsa$, and $r_{k+1} \leq \epsa$. 
        \item Let $\ell_i$ be the error of the hypothesis selected in the $i^{\text{th}}$ iteration. We have the guarantee that $\ell_i \leq \frac{2\eps}{r_{i}}$. 
    \end{enumerate}
    Then, using that $(r_i - r_{i+1})$ fraction of $S_{\test}$ is classified in the $i^{\text{th}}$ step, we can decompose the error of $H$ as\footnote{In the below, we are using that $H$ outputs $\bot$ on the remaining $r_{k+1}$ fraction of points not covered by one of the selected restrictions.}
    \begin{equation*}
        \error(H, S_\test) = r_{k+1} + \sum_{i = 1}^{k}(r_{i} - r_{i+1}) \cdot \ell_i.
    \end{equation*}
    Using the bounds $\ell_i \leq \frac{2\eps}{r_i}$ and $r_{k+1} \leq \epsa$, we have
    \begin{align*}
         \error(H, S_\test) &\leq \epsa + \sum_{i = 1}^k (r_{i} - r_{i+1}) \cdot \frac{2\eps}{r_i}\\
         &= \epsa + 2 \eps \cdot \sum_{i = 1}^k \paren*{1 - \frac{r_{i+1}}{r_i}}.
    \end{align*}
    Note that each term in the above summation is at most $1$. We can therefore remove the last term to obtain the upper bound,
    \begin{align*}
        \error(H, S_\test) &\leq \epsa + 2\eps + 2\eps \cdot \sum_{i = 1}^{k-1} \paren*{1 - \frac{r_{i+1}}{r_i}} \\
        &\leq \epsa + 2\eps + 2\eps \cdot \sum_{i = 1}^{k-1} \ln\paren*{\frac{r_i}{r_{i+1}}} \tag{$1-x \leq \ln(1/x)$}\\
        &= \epsa + 2\eps + 2\eps \cdot \ln\paren*{\frac{r_1}{r_k}} \leq \eps(2 + 2\ln(1/\epsa)) + \epsa. 
    \end{align*}
    Finally, we argue that this quantity is $O(\eps\cdot \log(1/\epsa) + \epsa)$. If $\epsa \leq 1/2$, this holds because $2 + 2\ln(1/\epsa) \leq O(\log(1/\epsa))$. Otherwise, if $\epsa > 1/2$, it holds because the error is upper bounded by $1$ which is $O(\epsa)$.
\end{proof}
We also give a simple analysis of the time complexity of $\FindSubcube$.
\begin{proposition}[$\FindSubcube$ is efficient]
    \label{prop:time-subcube-approx}
    Let $T$ be an upper bound on the time needed to evaluate any hypothesis in $\mcH$. Then, $\FindSubcube(S_{\test}, \eps, s)$ runs in $O(Ts\ln(1/\eps) \cdot |S_{\test}| \cdot n^{O(d)})$ time.
\end{proposition}
\begin{proof}
    The time is dominated by determining the loss of each restriction hypothesis in the loop. The loop runs for at most $O(2s \ln (1/\eps))$ iterations, there are $n^{O(d)}$ many restriction hypotheses, and determining the loss for one hypothesis takes time at most $T \cdot |S_{\test}|$. 
\end{proof}

\subsubsection{Finishing the proof of correctness}
Now we can put the above claims together and argue that {\sc PartitionLearn} finds
a hypothesis with small true distributional error in expectation over $\bS_{\train}$ and $\bS_{\test}$. 

\begin{proof}
There are at most $n^{O(ds')}$ subcube lists of length $s' \coloneqq \ceil{2s\ln(1/\eps)}$ and depth $d$, so by
\Cref{claim:emploss-closs-trueloss}, for any fixed $\mcH$ and $m_\test = O(\frac{ds' \ln n (\ln 1/\eps)^2}{\eps^2}) = O(\frac{ds \ln n (\ln 1/\eps)^3}{\eps^2})$, we have 
\begin{equation}
\label{eq:8}
\Ex_{\bS_\test}\bracket*{\sup_{\text{depth-$d$, length-$s'$ list $L$}}{\abs{\error(L \circ \mcH, D^\star_f) - \error(L \circ \mcH, \bS_\test)}}} \le \frac{\eps}{\log(1/\eps)}.
\end{equation}
By \Cref{cor:exists-good-partition}, the training phase produces a mapping $\bmcH$ is such that for the ground-truth partition $L^\star$, we have
\begin{equation}
\label{eq:9}
\Ex_{\bS_\train, A}[\error(L^\star \circ \bmcH, D^\star_f)] \le \frac{\eps}{\log(1/\eps)}.
\end{equation}
Putting together \Cref{eq:8} and \Cref{eq:9}, we have 
\begin{equation}
\Ex_{\bS_\test, \bS_\train, A}\bracket*{\error(L^\star \circ \mcH, \bS_\test)} \le \frac{2\eps}{\log(1/\eps)}.
\end{equation}

For a fixed $\mcH$ and $S_\test$, we will denote $\eps_\mcH \coloneqq \error(L^\star \circ \mcH, S_\test)$. 
By \Cref{lem:subcube-approx}, with parameter ``$\eps$'' set to $\eps_\mcH$ and parameter ``$\eps_\mathrm{additional}$'' set to $\eps$, we have that {\sc PartitionLearn} returns a hypothesis $L$ such that 
\[\error(L \circ \mcH, S_\test) \le O(\eps_\mcH \log(1/\eps) + \eps).\]

Putting it all together, we have
\begin{align*}
\Ex_{\bS_\test, \bS_\train, A}\bracket*{\error(L \circ \bmcH, \bS_\test)} &\le \Ex_{\bS_\test, \bS_\train, A}\bracket*{O\paren*{\boldsymbol{\eps}_{\bmcH} \log(1/\eps) + \eps}} \\
&\le \bracket*{O\paren*{\frac{3\eps}{\log(1/\eps)} \log(1/\eps)}} \\
&= O(\eps).
\end{align*}
Finally, another application of \Cref{eq:8}, gives that $\error(L \circ \bmcH,, D^{\star}_f)$ is, in expectation, at most $O(\eps) + \eps/\log(1/\eps) \leq O(\eps)$.
\end{proof}

\subsection{Running time}
Now we will analyze the running time of {\sc PartitionLearn}.

\begin{lemma}
		\label{lem:time-partition}
{\sc PartitionLearn} runs in time 
\[n^{O(d)} \cdot \paren*{t \log t + \frac{\log^4(1/\eps)}{\eps^2}}. \]
\end{lemma}

\begin{proof}
The training phase runs in time $t \cdot n^{O(d)}$, since it runs $A$ on each restriction of depth $\le d$.
The test phase, by \Cref{prop:time-subcube-approx},
runs in time $O(Ts\ln(1/\eps) \cdot |S_{\test}| \cdot n^{O(d)})$, where $T$ is a bound on the time
required to evaluate a hypothesis returned by $A$.
The size of $S_\test$ is $O(\frac{ds \log n \log^3(1/\eps)}{\eps^2})$,
and we have $T \le O(t \log t)$ (see \Cref{lemma:dtlearn-efficient}).
We can eliminate the factors of $d$, $s$, and $\log n$ since $d \le n$ and $s \le n^{O(d)}$, and the desired running time bound follows. 
\end{proof}


\fi

\section{Limitations of \cite{BLMT-lifting}}
\label{sec:BLMT limitations}
\subsection{Information-theoretic lower bounds for \cite{BLMT-lifting}'s approach}

Recall that the approach of \cite{BLMT-lifting} worked in two distinct steps.
\begin{enumerate}
    \item First, they use subcube-conditional queries to learn the unknown distribution $D^{\star}$. In more detail, they attempt to find to find a small $d$ for which $D^{\star}$ admits a depth-$d$ decomposition into distributions that are uniform on a subcube, and learn this decomposition.
    \item They use this decomposition of $D^{\star}$ to learn the unknown function $f$.
\end{enumerate}

We show that this first step is information-theoretically impossible to perform efficiently if we only have access to random samples, rather than subcube-conditional queries.
\begin{theorem}[Impossibility of determining the depth of a distribution using few random samples]
    \label{thm:BLMT-lowerbound}
    Let the base distributions, $\mcD$, consist of the uniform distribution and all of its restrictions. It takes $\Omega(2^{n/2})$ samples from $D^\star$ to, with success probability at least $2/3$, distinguish between
    \begin{enumerate}
        \item $D{\star}$ has a depth-$0$ decomposition into distributions in $\mcD$, meaning that $D^{\star}$ is the uniform distribution.
        \item $D^{\star}$ robustly requires a depth-$(d \coloneqq n - O(\log n))$ decomposition, meaning any $D'$ satisfying $\TV(\mcD^{\star}, D') \leq 1/3$ requires a depth $d$ decomposition into distributions in $\mcD$.
    \end{enumerate}
\end{theorem}

Our proof of \Cref{thm:BLMT-lowerbound} builds on the classic result that testing whether $D^{\star}$ is uniform or far from uniform takes $\Omega(2^{n/2})$ samples \citep{Pan08}. That classic result uses the following approach: Suppose we draw $\bS \subseteq \bits^n$ uniformly among all subsets of $\tfrac{2^n}{2}$ elements. Then, $D_{\bS} = \Unif(\bS)$ is guaranteed to be far from uniform, since $\bS$ contains only half of the domain. Furthermore, a sample of $m \leq 2^{n/2}$ points from $D_{\bS}$ is unlikely to have any collisions. The key observation is that, if no collisions occur, the algorithm cannot tell if it received a sample from the uniform distribution over $\bits^n$ or a sample from $D_{\bS}$, since $\bS$ is unknown to the algorithm.

One difficulty of directly applying this result is that $D_{\bS}$ may only require a depth-$1$ decomposition---for example, if $S$ contains $x$ iff $x_1 = 1$. The key novel technical step in the proof of \Cref{thm:BLMT-lowerbound} is the following lemma showing this is unlikely.
\begin{lemma}
    \label[lemma]{lem:BLMT-lb-depth}
    Let the base distributions, $\mcD$, be the uniform distribution and all of its restrictions, and $d = n - O(\log n)$. For $\bS$ be drawn uniformly over all size-$\frac{2^n}{2}$ subsets of $\bits^n$ and $D_{\bS} = \Unif(\bS)$,
    \begin{equation*}
        \Pr_{\bS}\bracket*{\TV(D_{\bS}, D') \leq 1/3\text{ for some }D'\text{with a depth-$d$ decomposition by $\mcD$}} \leq \frac{1}{100}.
    \end{equation*}
\end{lemma}
\textbf{Proof sketch:} A natural approach to proving \Cref{lem:BLMT-lb-depth} is a two step process.
\begin{enumerate}
    \item For any fixed depth-$d$ decision trees $T$, proving that it is very unlikely a $D'$ that is uniform over the leaves of $T$ is close to $D_{\bS}$.
    \item Union bounding over the $\approx n^{2^d}$ unique depth-$d$ decision trees $T$.
\end{enumerate}
This approach requires union bounding over a doubly exponential number of decision trees. To avoid this doubly exponential quantity, we observe that it is sufficient to prove that all possible depth-$d$ leaves have ``high error". This is because, the distance of a $D'$ to be $D_{\bS}$ can be decomposed into the sum of a corresponding distance measure over every leaf $\ell$ of $T$. Furthermore, there are at most $2^{O(n)}$ possible leaves, so we only need to union bound over a singly-exponential number of leaves.


This proof will use Hoeffding's concentration inequality, applied to random variables drawn \emph{without replacement}.
\begin{fact}[Hoeffding's inequality for samples drawn \emph{without replacement}, \citep{Hoe63}]
    \label[fact]{fact:Hoe-wor}
    Let $S \in [0,1]^m$ be a bounded finite population, and let $\bx_1, \ldots, \bx_n$ be drawn \emph{uniformly without replacement} from $S$, and $\bX$ their sum. Then,
    \begin{equation*}
        \Pr[\abs*{\bX - \mu} \geq n\eps] \leq 2e^{-2\eps^2n} \quad\quad\text{where}\quad\mu \coloneqq \Ex[\bX].
    \end{equation*}
\end{fact}

\begin{proof}[Proof of \Cref{lem:BLMT-lb-depth}]
    This proof is broken into two pieces. First, we show that, with high probability over the choice of $\bS$, the following holds: For all restrictions $\ell$ of depth $\leq d$,
    \begin{equation}
        \label{eq:count-concentrated}
        \textbf{count}(\ell) \in \bracket*{\frac{1}{3} \cdot 2^{n-\depth(\ell)}, \frac{2}{3} \cdot 2^{n-\depth(\ell)}},
    \end{equation}
    where $\textbf{count}(\ell)$ is the random variable counting the number of $x$ consistent with $\ell$ that are contained in $\bS$. Then, we show that if \Cref{eq:count-concentrated} holds for all $\ell$ of depth $\leq d$, then $D_{\bS}$ is far from all distributions $D'$ with decompositions of depth at most $d$.

    Consider any fixed $\ell$ of depth at most $d$. Then, since $\bS$ is drawn uniformly among all size-$\tfrac{2^n}{2}$ subsets of $\bits^n$, the indicators $\Ind[x \in \bS]$ for each $x \in \bits^n$ are drawn uniformly without replacement from a population containing $\tfrac{2^n}{2}$ many $1$s and $\tfrac{2^n}{2}$ many $0$s. Therefore, applying \Cref{fact:Hoe-wor} with $\eps = 1/6$,
    \begin{equation*}
        \Pr\bracket*{\textbf{count}(\ell) \notin \bracket*{\frac{1}{3} \cdot 2^{n-\depth(\ell)}, \frac{2}{3} \cdot 2^{n-\depth(\ell)}}} \leq 2\exp\paren*{-\frac{2 \cdot 2^{n-\depth(\ell)}}{36}} \leq 2\exp\paren*{-\frac{2 \cdot 2^{n-d}}{36}}.
    \end{equation*}
    For our choice of $d = n - O(\log n)$, the above probability is less than $3^{-n}/100$. There are at most $3^n$ restrictions $\ell$, since every restriction can be specified by a string in $\sset{\pm 1, \star}^n$, so a union bound gives that \Cref{eq:count-concentrated} holds for all $\ell$ with depth $\leq d$ with probability at least $0.99$.

    Next, we show that if \Cref{eq:count-concentrated} holds for all such $\ell$, then $D_{\bS}$ is far from all $D'$ with depth at most $d$. Consider any such $D'$. Then, there is some decision tree $T$ of depth at most $d$ such that the marginal distribution of $D'$ on any leaf of $T$ is uniform. We can decompose the TV distance between $D_{\bS}$ and $D'$ as a summation over the leaves of $T$,
    \begin{equation*}
        \TV(D_{\bS}, D') = \frac{1}{2} \cdot \sum_{x \in \bits^n} \abs*{D_{\bS}(x) - D'(x)} =\sum_{\text{leaves }\ell \in T} \underbrace{\frac{1}{2} \cdot \sum_{x \text{ consistent with }\ell} \abs*{D_{\bS}(x) - D'(x)}}_{\coloneqq \error(\ell)}.
    \end{equation*}
    Since $D'$ is uniform over $\ell$, $D'(x)$ must be constant on every $x$ consistent with $\ell$. Furthermore, $D_{\bS}(x) = \frac{2}{2^n} \cdot \Ind[x \in \bS]$ since $D_{\bS}$ is uniform over all $\frac{2^n}{2}$ points contained in $\bS$. Therefore,
    \begin{align*}
        \error(\ell) &\geq \frac{1}{2} \cdot\min_{p \in [0,1]}\sset*{\textbf{count}(\ell) \cdot \abs*{p - \frac{2}{2^n}} + (2^{n-\depth(\ell)} -\textbf{count}(\ell)) \cdot p}.\\
        &=  \min\paren*{\textbf{count}(\ell), 2^{n-\depth(\ell)} - \textbf{count}(\ell)} \cdot \frac{1}{2^n}.
    \end{align*}
    Conditioned in \Cref{eq:count-concentrated} holding, we therefore have that
    \begin{equation*}
        \TV(D_{\bS}, D') = \sum_{\ell \in T} \error(\ell) \geq \sum_{\ell \in T}\frac{2^{n - \depth(\ell)}}{3 \cdot 2^n} =\frac{1}{3} \cdot \sum_{\ell \in T} 2^{-\depth(\ell)}.
    \end{equation*}
    Finally, we observe that every $x \in \bits^n$ must be covered by exactly one leaf $\ell \in T$, and each $\ell \in T$ covers $2^{n-\depth(\ell)}$ unique $x$. Therefore $ \sum_{\ell \in T} 2^{-\depth(\ell)} = 1$. Combining this with the above gives the desired result.
\end{proof}

We are now ready to prove the main result of this section.
\begin{proof}[Proof of \Cref{thm:BLMT-lowerbound}]
    Assume, for the sake of contradiction, there exists an algorithm $A$ that uses $m = \frac{2^{n/2}}{20}$ samples with the following guarantee: If its samples are drawn iid from $\Unif(\bits^n)$, it outputs $1$ with probability at least $2/3$. If its samples are drawn iid from some $D^{\star}$ that robustly requires a depth-$(d \coloneqq n - \Omega(\log n))$ decomposition, then $A$ outputs $0$ with probability at least $2/3$. Define,
    \begin{equation*}
        p \coloneqq \Ex[A(\bx)] \quad\quad\text{where $\bx_1, \ldots, \bx_m$ be drawn uniformly \emph{without replacement} from $\bits^n$}.
    \end{equation*}
    We consider two cases based on whether $p$ is large or small and show $A$ fails in both.

    \pparagraph{Case 1:} $(p \leq 1/2)$. In this case, we will show that when its input distribution truly is uniform, $A$ does not correctly output $1$ with probability at least $2/3$. For any $y_1, \ldots, y_m$, let $\mathrm{collision}(y)$ indicate whether there is some $i \neq j$ for which $y_i = y_j$. Then,
    \begin{equation*}
        \Prx_{\by \sim \Unif(\bits^n)^m}\bracket*{A(\by) = 1} \leq \Prx_{\by}[\mathrm{collision}(\by) = 1] + \Prx_{\by}[A(\by) = 1 \mid \mathrm{collision}(\by) = 0].
    \end{equation*}
    The probability that $\by_i = \by_j$ for any $i \neq j$ is $\frac{1}{2^n}$. By a union bound over all $\binom{m}{2}$ choices of $i\neq j$ for which $y_i$ could be equal to $y_j$, we see that
    \begin{equation*}
        \Prx_{\by}[\mathrm{collision}(\by) = 1] \leq \frac{\binom{m}{2}}{2^n} \leq 0.01.
    \end{equation*}
    To bound the other term, we observe by symmetry that, conditioned on them being unique, the distribution of $\by_1, \ldots, \by_m$ is exactly that of $m$ samples drawn uniformly \emph{without replacement} from $\bits^n$. Therefore, $\Pr_{\by}[A(\by) = 1 \mid \mathrm{collision}(\by) = 0] = p$. Combining these we have that
    \begin{equation*}
         \Prx_{\by \iid \Unif(\bits^n)}\bracket*{A(\by) = 1}  \leq 0.01 + p \leq 0.51.
    \end{equation*}
    As desired, we show that $A$ failed in this case.
    
    \pparagraph{Case 2:} $(p > 1/2)$. In this case, we will show there is some $D^{\star}$ that robustly requires a depth-$d$ decomposition for which $A$ does not output $0$ with probability at least $2/3$. Suppose we drew $\bS$ uniformly without and then draw $\by_1, \ldots, \by_m \iid D_{\bS} = \Unif(\bS)$. Then, as before,
    \begin{equation*}
        \Pr_{\bS, \by \sim (D_{\bS})^m}[A(\by) = 0] \leq \Prx_{\by}[\mathrm{collision}(\by) = 1] + \Prx_{\by}[A(\by) = 0 \mid \mathrm{collision}(\by) = 0].
    \end{equation*}
    Here, we have a similar bound on the probability of a collision. In this case, the probability that $\by_i = \by_j$ is $\frac{1}{|\bS|} = \frac{2}{2^n}$, and therefore by a similar union bound,
    \begin{equation*}
        \Prx_{\by}[\mathrm{collision}(\by) = 1] \leq \frac{2\cdot \binom{m}{2}}{2^n} \leq 0.01.
    \end{equation*}
    Once again, conditioned on them being unique the distribution of $\by_1, \ldots, \by_m$ is exactly that of $m$ samples drawn uniformly \emph{without replacement} from $\bits^n$. Therefore, $\Pr_{\by}[A(\by) = 0 \mid \mathrm{collision}(\by) = 0] = 1-p \leq 0.5$. Combining with the above, we have that,
    \begin{equation*}
        \Pr_{\bS, \by \sim (D_{\bS})^m}[A(\by) = 0]  \leq 0.51.
    \end{equation*}
    By \Cref{lem:BLMT-lb-depth}, with probability at least $0.99$, $D_{\bS}$ robustly requires a depth-$d$ decomposition. Therefore, if it were the case that on every such distribution, $A$ outputs $0$ with probability at least $2/3$, it would also be the case that
    \begin{equation*}
        \Pr_{\bS, \by \sim (D_{\bS})^m}[A(\by) = 0]  \geq 2/3 \cdot \Pr[\text{$D_{\bS}$ robustly requires a depth-$d$ decomposition}] =0.66.
    \end{equation*}
    We have arrived at a contradiction, and therefore $A$ must fail on at least one such distribution.
\end{proof}

\subsection{Why \cite{BLMT-lifting} can't lift robust learners}
We briefly summarize why the approach of \cite{BLMT-lifting} cannot lift robust learners. Recall their approach first decomposes $D^{\star}$ into uniform distributions on distinct leaves of a decision tree, and then runs the base learner on each of these uniform distributions. The crux of the issue is that their algorithm for learning $D^{\star}$ is not robust. In more detail, given samples $(D^{\star})'$ that is close, in TV distance, to a $D^{\star}$ which has a depth-$d$ decomposition into uniform distributions, their algorithm is \textsl{not} guaranteed to produce a tree with leaves on which $(D^{\star})'$ is close to uniform. Hence, even if the base learner is robust, it might not be given inputs with distributions that are close to uniform.

The reason their decomposition algorithm is not robust is that they need to estimate, for each coordinate $i \in [n]$, a quantity called the ``influence" of $i$ on the distribution $(D^{\star})'$. These estimates need to be computed to very small error (asymptotically smaller than the error of the resulting decomposition). Therefore, by only corrupting $\mcD^{\star}$ by a small amount, an adversary can force these estimates to err but a relatively large amount and cause \cite{BLMT-lifting}'s lifter to fail.

\subsection{Why \cite{BLMT-lifting} can only handle uniform base distributions}

Another downside of \cite{BLMT-lifting}'s approach is that they can only decompose distributions into uniform base distributions. For example, their decomposition algorithm fails if the base distributions are $\mcD_{\mathrm{prod}} \coloneqq \sset{\text{all product distributions}}$. The issue here is, once again, the use of influences in \cite{BLMT-lifting}'s algorithm. They use the influence of a coordinate $i \in [n]$ as a proxy for whether $x_i$ must be queried to decompose $D^{\star}$ into uniform pieces. This is justified by the fact that, if the base distributions are uniform, the only coordinates with nonzero influence will be coordinates queried by the true decomposition of $D^{\star}$. However, this is no longer true if the base distributions are $\mcD_{\mathrm{prod}}$. Even when $d = 0$ (i.e. $D^{\star}$ is itself just a product distribution), all $n$ coordinates can have high influence. Therefore, influence no longer provides a good proxy for what to query.






\section*{Acknowledgments}

We thank Adam Klivans as well as the COLT reviewers for helpful feedback.

Guy, Carmen, and Li-Yang are supported by NSF awards 1942123, 2211237, 2224246, a Sloan Research Fellowship, and a Google Research Scholar Award. Guy is also supported by a Jane Street Graduate Research Fellowship and Carmen by an NSF GRFP. Jane is supported by NSF awards 2006664 and 310818 and an NSF GRFP.

\bibliographystyle{alpha}
\bibliography{ref}

\end{document}